\documentclass[11pt, letterpaper, logo, onecolumn]{googledeepmind}

\pdftrailerid{redacted}

\makeatletter
\renewcommand\bibentry[1]{\nocite{#1}{\frenchspacing\@nameuse{BR@r@#1\@extra@b@citeb}}}
\makeatother
\usepackage{kantlipsum, lipsum}
\usepackage{dsfont}
\usepackage{amsthm}
\usepackage{thmtools}
\usepackage{amsmath}
\usepackage{thm-restate}
\usepackage{algorithm}
\usepackage{algorithmicx}
\usepackage{algpseudocode}
\usepackage{microtype}
\usepackage{graphicx}
\usepackage{booktabs} %
\usepackage{float}
\usepackage{nicefrac}
\usepackage{todonotes}
\usepackage{tikz}
\usepackage{xcolor}
\DeclareRobustCommand\encirclea{\tikz[baseline=(char.base)]{\node[shape=circle,fill,inner sep=0.5pt, minimum size=8pt] (char) {\textcolor{white}{1}}}}
\DeclareRobustCommand\encircleb{\tikz[baseline=(char.base)]{\node[shape=circle,fill,inner sep=0.5pt, minimum size=8pt] (char) {\textcolor{white}{2}}}}
\usepackage{hyperref}[citecolor=magenta]
\hypersetup{
    colorlinks = true,
    citecolor = {magenta},
}
\usepackage[most,skins,theorems]{tcolorbox}
\tcbset{
  aibox/.style={
    width=\linewidth,
    top=10pt,
    bottom=4pt,
    colback=blue!6!white,
    colframe=black,
    colbacktitle=black,
    enhanced,
    center,
    attach boxed title to top left={yshift=-0.1in,xshift=0.15in},
    boxed title style={boxrule=0pt,colframe=white,},
  }
}
\newtcolorbox{AIbox}[2][]{aibox,title=#2,#1}
\definecolor{lightblue}{rgb}{0.22,0.45,0.70}%
\usepackage{url}
\usepackage{amsmath}
\usepackage{amssymb}
\usepackage{mathtools}
\usepackage{mathrsfs}
\usepackage{nicefrac}
\usepackage{dsfont}
\usepackage{enumitem}
\usepackage{graphicx}
\usepackage{float}
\usepackage[utf8]{inputenc} %
\usepackage[T1]{fontenc}    %
\usepackage{hyperref}       %
\usepackage{url}            %
\usepackage{booktabs}       %
\usepackage{amsfonts}       %
\usepackage{nicefrac}       %
\usepackage{microtype}      %
\usepackage{xcolor}         %
\usepackage{graphicx}
\usepackage{subcaption}
\usepackage{amssymb}
\usepackage{wrapfig}
\usepackage{lipsum}
\usepackage{enumitem}
\usepackage{stackengine}
\usepackage[font=small,labelfont=bf]{caption}
\usepackage{color}
\usepackage{adjustbox}
\usepackage{newtxmath}

\usepackage{amsmath,amsfonts,bm}

\def\Figref#1{Figure~\ref{#1}}

\def\Secref#1{Sec.~\ref{#1}}
\def\Appref#1{App.~\ref{#1}}

\def\TwoSecrefs#1#2{Secs. \ref{#1} and \ref{#2}}

\def\eqref#1{equation~\ref{#1}}
\def\Eqref#1{Equation~\ref{#1}}

\def\1{\bm{1}}

\DeclareMathAlphabet{\mathsfit}{\encodingdefault}{\sfdefault}{m}{sl}
\SetMathAlphabet{\mathsfit}{bold}{\encodingdefault}{\sfdefault}{bx}{n}

\newcommand{\E}{\mathbb{E}}

\newcommand{\R}{\mathbb{R}}

\newcommand{\gsim}{\raisebox{-0.13cm}{~\shortstack{$>$ \\[-0.07cm] $\sim$}}~}

\newcommand{\bv}{\mathbf{v}}

\newcommand{\by}{\mathbf{y}}
\newcommand{\bx}{\mathbf{x}}

\newcommand{\bw}{\mathbf{w}}

\newcommand{\bs}{\mathbf{s}}

\def\Eqref#1{Eq.~\ref{#1}}
\def\Figref#1{Fig.~\ref{#1}}

\definecolor{blanchedalmond}{rgb}{1.0, 0.92, 0.8}
\definecolor{carmine}{rgb}{0.59, 0.0, 0.09}
\definecolor{lightblue}{rgb}{0.22,0.45,0.70}%

\renewcommand{\mathbf}{\boldsymbol}

\makeatletter
\def\Ddots{\mathinner{\mkern1mu\raise\p@
\vbox{\kern7\p@\hbox{.}}\mkern2mu
\raise4\p@\hbox{.}\mkern2mu\raise7\p@\hbox{.}\mkern1mu}}
\makeatother

\newcommand{\abs}[1]{\left| #1 \right|}

\newcommand{\paren}[1]{\left( #1 \right)}

\newcommand{\Rex}{\mathrm{Rex}}

\newcommand{\brck}[1]{\left [ #1 \right ] }

\newcommand{\methodname}{PAV}
\newcommand{\Exp}{\mathbb{E}}

\newcommand{\btheta}{{\mathbf{\theta}}}

\makeatletter
\newcommand{\subalign}[1]{%
  \vcenter{%
    \Let@ \restore@math@cr \default@tag
    \baselineskip\fontdimen10 \scriptfont\tw@
    \advance\baselineskip\fontdimen12 \scriptfont\tw@
    \lineskip\thr@@\fontdimen8 \scriptfont\thr@@
    \lineskiplimit\lineskip
    \ialign{\hfil$\m@th\scriptstyle##$&$\m@th\scriptstyle{}##$\hfil\crcr
      #1\crcr
    }%
  }%
}
\makeatother

\newtheorem{theorem}{Theorem}[section]

\newtheorem{lemma}{Lemma}[section]
\newtheorem{proposition}{Proposition}[section]

\newtheorem{remark}{Remark}[section]

\usepackage[authoryear, sort&compress, round]{natbib}

\title{Rewarding Progress:  Scaling Automated Process Verifiers for LLM Reasoning}

\correspondingauthor{asetlur@cs.cmu.edu}

\author[1,3,*]{Amrith Setlur}
\author[1,*]{Chirag Nagpal}
\author[2]{Adam Fisch}
\author[2]{Xinyang Geng}
\author[2]{Jacob Eisenstein}
\author[2]{Rishabh Agarwal}
\author[1]{\\Alekh Agarwal}
\author[$\dagger$,2]{Jonathan Berant}
\author[$\dagger$,2,3]{Aviral Kumar}

\affil[1]{Google Research}
\affil[2]{Google DeepMind}
\affil[3]{Carnegie Mellon University}
\affil[*]{Equal contribution}
\affil[$\dagger$]{Equal advising}

\begin{abstract}
A promising approach for improving reasoning in large language models is to use process reward models (PRMs). PRMs provide feedback at each step of a multi-step reasoning trace, potentially improving credit assignment over outcome reward models (ORMs) that only provide feedback at the final step. However, collecting dense, per-step human labels is not scalable, and training PRMs from automatically-labeled data has thus far led to limited gains. To improve a \emph{base} policy by running search against a PRM or using it as dense rewards for reinforcement learning (RL), we ask: ``How should we design process rewards?''. Our key insight is that, to be effective, the process reward for a step should measure \emph{progress}: a change in the likelihood of producing a correct response in the future, before and after taking the step, corresponding to the notion of step-level advantages in RL. Crucially, this progress should be measured under a \emph{prover} policy distinct from the base policy. We theoretically characterize the set of good provers and our results show that optimizing process rewards from such provers improves exploration during test-time search and online RL. In fact, our characterization shows that weak prover policies can substantially improve a stronger base policy, which we also observe empirically. We validate our claims by training \emph{process advantage verifiers (PAVs)} to predict progress under such provers, and show that compared to ORMs, test-time search against PAVs is $>8\%$ more accurate, and  $1.5-5\times$ more compute-efficient. Online RL with dense rewards from PAVs enables \emph{one of the first results} with  $5-6\times$ gain in sample efficiency, and $>6\%$ gain in accuracy, over ORMs.
\end{abstract}

\begin{document}

\maketitle

\vspace{-0.2cm}
\section{Introduction}
\label{sec:introduction}
\vspace{-0.2cm}

Trained reward models or \emph{verifiers} are often used to improve math reasoning in large language models, either by re-ranking solutions at test-time~\citep{collins2000discriminative} or via reinforcement learning (RL)~\citep{uesato2022solving}. Typically, verifiers are  trained to predict the outcome of an entire reasoning trace, often referred to as \emph{outcome} reward models (ORM)~\citep{cobbe2021training,hosseini2024v}. However, ORMs only provide a sparse signal of correctness, which can be hard to learn from and inefficient to search against. 
This challenge is alleviated by fine-grained supervision, in theory. For reasoning, prior works train \emph{process} reward models (PRMs) that assign intermediate rewards after each step of search~\citep{snell2024scaling} or during RL. While \citet{lightman2023let} obtains PRM annotations from human raters, this approach is not scalable. More recent works~\citep{wang2024mathshepherd,luo2024improve} train PRMs to predict automatically-generated annotations that estimate future success of solving the problem, akin to value functions in RL. 
So far, automated PRMs, especially as dense rewards in RL, only improve by 1-2\% over ORMs~\citep{shao2024deepseekmath}, raising serious doubts over their utility.

To resolve these uncertainties, in this paper, we train  PRMs with automated annotations, such that optimizing the dense rewards from trained PRMs can improve a \emph{base} policy compute- and sample-efficiently, during test-time search and online RL. 
For this, we first ask: \textbf{(i)} what should the \textit{per-step} process rewards measure, and \textbf{(ii)} what kind of automated data collection strategy should we use to train PRMs that predict this measure. For \textbf{(i)}, conventional belief~\citep{lightman2023let, uesato2022solving} has been to measure mathematical correctness or relevance of steps. 
But, it is unclear if this supervision yields the most improvement in the base policy (\textit{e.g.}, a policy may need to generate simpler, repetitive, and even incorrect steps to explore and discover the final answer during test-time search and RL). 
\textbf{Our key insight} is that per-step, process rewards that measure a notion of \emph{progress}: change in the likelihood of arriving at a correct final answer before and after taking the step, are  effective, for both test-time beam search and online RL. Reinforcing steps that make progress regardless of whether they appear in a correct or incorrect trace diversifies the \textbf{\emph{exploration}} of possible answers at initial steps, which is crucial when the approach to solve a problem is not clear. Formally, such  rewards correspond to per-step \emph{advantages} of steps from the RL
literature~\citep{suttonrlbook}. We empirically show that using advantages in addition to ORM rewards outperforms the typical use of future probabilities of success or $Q$-values~\citep{wang2024mathshepherd} for both search and RL. This is because, when given a combinatorial space of responses, under bounded computational and sampling constraints, $Q$-values mainly ``exploit'' states whereas advantages also ``explore'' steps that make the most progress towards the final answer (\Figref{fig:illustration_panel}).

\begin{figure}[t]
    \centering
    \includegraphics[width=0.8\linewidth]{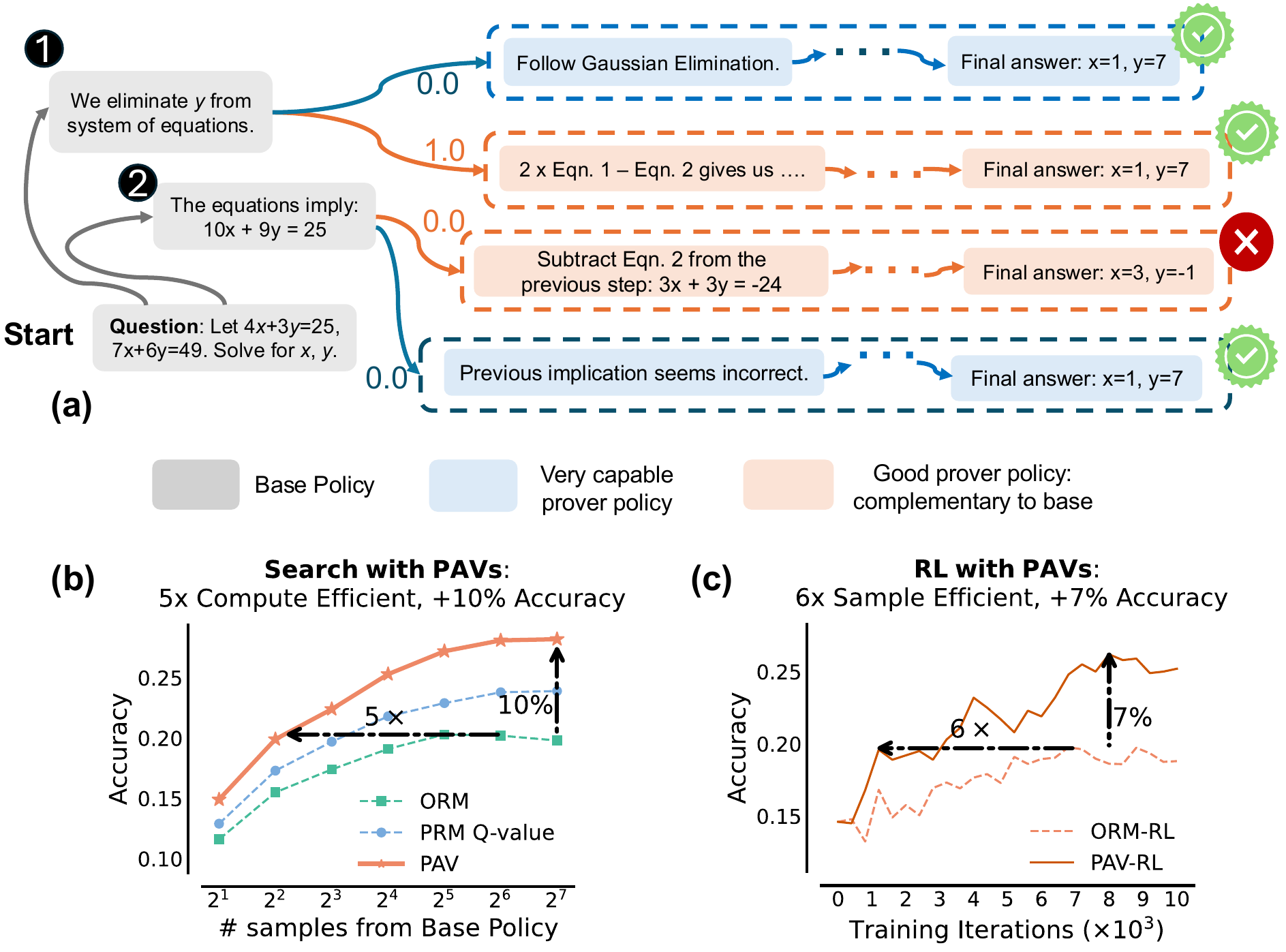}
    \vspace{-0.2cm}
    \caption{
    \textbf{\emph{Process advantage verifiers (PAV):}} 
    Process reward for a step is defined as progress (advantage) under the prover policy,  \textit{i.e.}, change in prover policy's success rate before and after the step.
    \textbf{(a):} The base policy samples both correct \encirclea ~and incorrect \encircleb ~steps but struggles to succeed from either. A strong prover policy completes the solution from both steps, and is unable to adequately reflect progress made by \encirclea~and  \encircleb~(both scored 0.0). Conversely, a complementary prover policy distinguishes \encirclea, \encircleb~ more prominently (only succeeds from \encirclea). \textbf{(b,c):} Compared to ORMs, PAVs are 5x more compute efficient, 10\% more accurate in test-time search, and 6x more sample efficient, 7\% more accurate for online reinforcement learning (RL).
}
    \vspace{-0.2cm}
    \label{fig:intro-figure}
\end{figure}

To answer \textbf{(ii)}, we first note that  advantages under a poor base policy are $\approx 0$ on most steps, and thus will not be informative for search or RL. In addition, regardless of the strength of the base policy, using its own per-step advantages as process rewards in RL will result in base policy updates equivalent to \emph{only} using outcome rewards for RL (since a standard policy gradient algorithm already computes advantages). 
Hence, we propose to use advantages estimated via rollouts under a different \textbf{\emph{prover policy}} as process rewards (\Figref{fig:intro-figure}(a)). How should we choose this prover policy? A natural guess would be to use a very capable prover. However, we show advantages under an overly capable prover policy, that can succeed from any step, fail to distinguish good and bad steps. A similar argument holds for very weak provers. 

In theory, we formalize this intuition to define good provers as policies that are \emph{complementary} to the base policy (\textit{i.e.}, policies with advantages that can contrast steps produced by the base policy sufficiently), while still producing step-level advantages correlated with those of the base policy. For \textit{e.g.}, for Best-of-$K$ policies~\citep{nakano2021webgpt} corresponding to a base policy, we empirically find that provers corresponding to $K>1$ (but not too large) are more capable at improving the base policy. Contrary to intuition,  the set of complementary provers also contains policies that are worse than the base policy. 
To predict the advantages of such provers we train dense verifiers, called \emph{\textbf{process advantage verifiers (PAVs)}}, 
that accelerate sample and compute efficiency of RL and search.

With the conceptual design of PAVs in place, we prescribe practical workflows for training PAVs and demonstrate their efficacy on a series of 2B, 9B, and 27B Gemma2 models~\citep{team2024gemma}. 
PAV training data is gathered by sampling ``seed'' solution traces from the prover and partial rollouts from the same to estimate the $Q$-value at each prefix of the seed trace. Our workflow prescribes favorable ratios for seed and partial rollouts.
Our first set of empirical results show that for an equal budget on test-time compute, beam search against trained PAVs is >8\% better in accuracy, and \mathbf{$1.5-5\times$} more compute efficient compared to re-ranking complete traces against an ORM (\Figref{fig:intro-figure}(b)). Dense rewards from PAVs improve the efficiency of step-level exploration during search by pruning the combinatorial space of solutions aggressively and honing in on a diverse set of possible sequences. Finally, we demonstrate \textbf{\emph{for the first time}}, that using PAVs as dense rewards in RL scales up data efficiency by \mathbf{$6\times$} compared to only using outcome rewards (\Figref{fig:intro-figure}(c)). Moreover, base policies trained with PAVs also achieve \mathbf{$8\times$} better Pass @$N$ performance (probability of sampling the correct solution in $N$ attempts), and consequently afford a higher ceiling on  the performance of any test-time re-ranker. Finally, running RL with PAVs discovers solutions to hard problems that sampling from the SFT policy with a very large budget can't solve.

\vspace{-0.2cm}
\section{Preliminaries, Definitions, and Notation}
\label{sec:prelim}
\vspace{-0.2cm}

Following protocols from \citet{uesato2022solving,lightman2023let}, a reasoning trace from an LLM consists of multiple logical steps separated by a demarcation token. 
An outcome reward model (ORM) is a trained verifier that assigns a  numerical score after the last step of the trace, and a process reward model (PRM) is a  trained verifier that scores each step of the trace individually. 

\textbf{Problem setup and notation.} Given a math problem $\bx \in \mathcal{X}$, our goal is to improve a \emph{base policy} $\pi$ that samples a response $\by \sim \pi(\cdot \mid \bx)$ in the set $\mathcal{Y}$. 
A response $\by$ consists of multiple reasoning steps (maximum $H$), separated by a delimiter (`next line' in our case), \textit{i.e.}, $\by = (a_1, a_2, \ldots, a_H)$. 
Since sampling is auto-regressive, we can view each step as an action taken by the agent $\pi$ in a Markov decision process (MDP) with deterministic dynamics. Specifically, we treat  the prefix $(\bx, a_1, \ldots, a_{h-1})$ 
as the current \emph{state} $\bs_h$ and next step $a_{h} \sim \pi(\cdot \mid \bx)$ as the \emph{action} taken by $\pi$ at $\bs_h$, resulting in the next state $\bs_{h+1}$. For problem $\bx$, with ground-truth response 
$\by^{\star}_{\bx}$, we can evaluate the accuracy of $\pi$ by running a regular expression match on the final answer~\citep{hendrycksmath2021}: $\Rex (\by, \by^{\star}_{\bx})  \mapsto \{0,1\}$, \textit{i.e.}, accuracy is given by $\mathbb{E}_{\by \sim \pi(\cdot \mid \bx)} \left[\Rex(\by, \by^{\star}_{\bx})\right]$. Now, given a dataset $\mathcal{D}=\{(\bx_i, \by_{\bx_i}^\star)\}_{i}$ of problem-solution pairs, the main goal is to learn a good base policy by optimizing this outcome reward on $\mathcal{D}$. Next, we see how we can leverage the final answer verifier $\Rex$  available on  $\mathcal{D}$ to train ORMs and PRMs.

\textbf{Outcome reward model (ORM).} 
Given a response $\by$, an ORM estimates the ground-truth correctness $\Rex(\by, \by^{\star}_{\bx})$. To train such a model we first take problems in $\mathcal{D}$, and collect training data of the form $\{(\bx, \by \sim \pi(\cdot \mid \bx), \Rex(\by, \by^{\star}_{\bx}))\}$. Then we train an ORM that takes as input a problem-response pair $(\bx, \by)$ and predicts $\Rex(\by, \by^{\star}_{\bx})$. At test time, when $\by^{\star}_{\bx}$ is unknown, the ORM is used to score candidate solutions revealed by test-time search.
 Given a base policy $\pi$, a {Best-of-$K$} policy: $\mathrm{BoK}(\pi)$, is a policy that samples $K$ responses from $\pi$, scores them against an ORM, and returns the one with the highest score. Whenever the ORM matches $\Rex$, the performance of $\mathrm{BoK}(\pi)$ is referred to as Pass @$K$.
Furthermore, when the likelihood of $\pi$ solving problem $\bx$ is $p_{\bx}$, then for $\mathrm{BoK}(\pi)$ this likelihood is given by the expression: $1-(1-p_{\bx})^K$. In general, this is larger than $p_{\bx}$, making  $\mathrm{BoK}(\pi)$  stronger than $\pi$ for $K>1$.

\textbf{Standard process reward models (PRMs).} A PRM scores every step $a_h$ in a  multi-step response $\by \sim \pi$ (\textit{e.g.}, in \cite{lightman2023let} PRMs are trained to score correct steps over  incorrect and irrelevant ones).
But, unlike ORMs, which only require $\Rex$ for data collection, PRM training data requires expensive step-level human annotations. Prior works~\citep{wang2024mathshepherd,luo2024improve} attempted to scale process rewards automatically by sampling from the model to provide a heuristic understanding of when a step is actually correct. 
In particular, they evaluate a prefix by computing the expected future accuracy of multiple completions sampled from $\pi$, after conditioning on the prefix, \textit{i.e.}, value function $Q^\pi$  (\Eqref{eq:value_function}) from RL. 
Similarly, we define  $V^\pi(\bs_h) \coloneqq \Exp_{a_h\sim \pi(\cdot \mid \bs_h)} Q^\pi(\bs_h, a_h)$  
as value of state $\bs_h$. These works use $Q^\pi$ as the PRM that assigns a score of $Q^\pi(\bs_h,a_{h})$ to the action $a_h$, at state $\bs_h$.  
{
\setlength{\abovedisplayskip}{12pt}
\setlength{\belowdisplayskip}{8pt}
\begin{align}
    \label{eq:value_function}
    Q^{\pi}(\underbrace{(\bx, a_1, \ldots , a_{h-1})}_{\text{state}~ \bs_{h}}, \underbrace{a_{h}}_{\text{action } a_{h}} ) =\underbrace{
    {\Exp}_{
    {a_{h+1}, \ldots, a_H \sim {\pi}(\cdot| \bs_{h}, a_h)}
    } \Big[ \Rex\left( (a_1, \ldots, a_H), \by^{\star}_{\bx}\right)\Big]}_{\text{likelihood of future  success}},
\end{align}
}%

\textbf{Using PRMs for beam search at test-time.} Given a PRM, a natural way to spend test-time compute is to use it as a step-level re-ranker within a beam search procedure~\citep{snell2024scaling}. For each problem,  at step $0$, a beam of maximum width $B$, is initialized with a single state consisting of just the problem.
At step $h$, a beam contains partial responses unrolled  till a set of states or prefixes $\{\bs_{i}\}_{i=1}^B$. From each state $\bs_{i}$ in this set, $C$ independent actions or steps $\{a_{i, j}\}_{j=1}^C$ are sampled from $\pi(\cdot \mid \bs_{i})$, each of which leads to a new state.  Process rewards from PRMs assign a score to every new state $(\bs_{i}, a_{i,j})$, and only the states corresponding to the top $B$ values are retained in the beam for the next step. 

\vspace{-0.2cm}
\section{How Should we Define Process Rewards and Why?}
\label{sec:value_functions}
\vspace{-0.2cm}

Ultimately, we are interested in test-time search and RL methods that can most efficiently and reliably discover solution traces with the correct final answer, thus maximizing  $\Rex$. To this end, \emph{process rewards should serve as step-level supervision to indirectly maximize outcome-level $\Rex$}. Our position contrasts with conventional belief that process rewards should mainly evaluate mathematical correctness or relevance of individual       steps~\citep{lightman2023let,uesato2022solving}, since LLMs might need to generate trivial or repetitive intermediate steps in order to discover a trace with the correct final answer.  
With this insight, in this section we approach the design of dense automated step rewards as a form of supervision to be used in conjunction with sparse outcome rewards to improve the base policy.  \looseness=-1

In an MDP, a starting point to design step-level dense feedback that is eventually meant to optimize a sparse outcome reward $\Rex$ is to consider the notion of a \emph{potential function}~\citep{ng1999policy}: in our case, this is a function that summarizes the difference between some statistic of the policy at the future state and the same statistic computed at the current state.
By appealing to this framework, in \Secref{subsec:adv_not_value}, we show that advantages -- not value functions~\citep{wang2024mathshepherd,luo2024improve} -- that measure a notion of ``progress'' at each new step are more appropriate for use as dense rewards in search and RL (primarily for exploration). Then in \TwoSecrefs{subsec:stylized_problem} {subsec:theory}, we show that this progress or advantage vakue is measured best under a policy $\mu$, different from the base policy $\pi$. We call this policy $\mu$, the \textbf{\emph{prover policy}}.

 \begin{figure}
    \centering
    \includegraphics[width=0.99\linewidth]{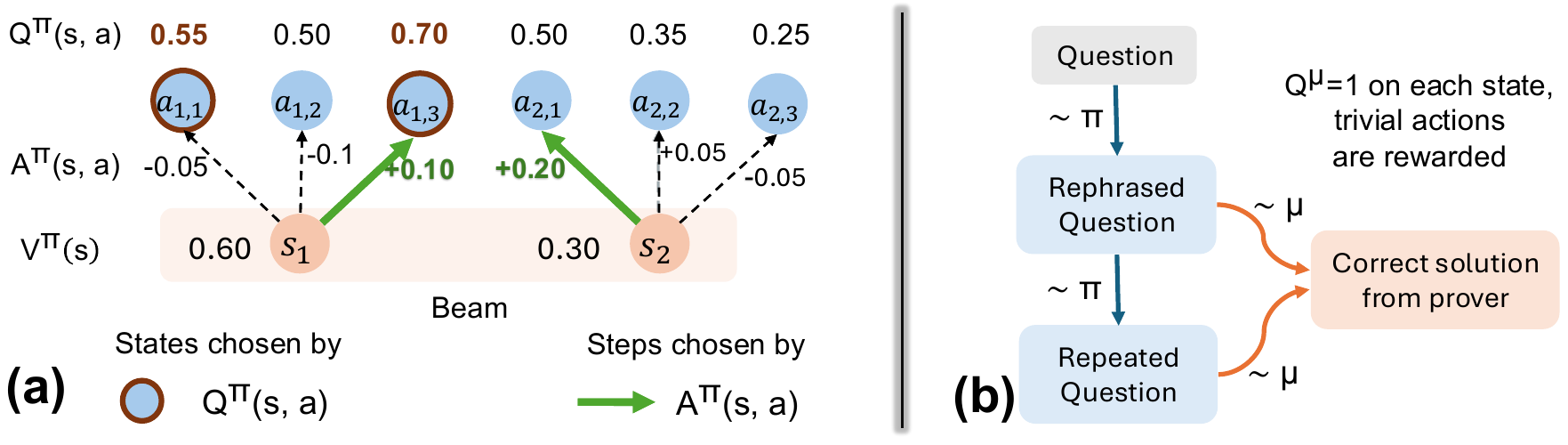}
    \caption{\textbf{\textit{Issues with using $Q$-values as process rewards}:} \textbf{(a):}   Unlike $A^\pi$, $Q^\pi$ mixes action evaluation with the $Q$-value of the previous state. Beam search with $Q^\pi$ exploits high-likelihood states, while adding $A^\pi$ (\textit{e.g.}, $Q^\pi + \alpha A^\pi$ in \Eqref{eq:prm_rl_grad}) aids in exploring states reached by making actions that induce progress, \textit{i.e.}, increase likelihood of success. 
    \textbf{(b):} $Q^\mu$ from a strong prover $\mu$ can assign unmerited bonuses to trivial actions.}
    \vspace{-0.2cm}
    \label{fig:illustration_panel}
\end{figure}

\vspace{-0.2cm}
\subsection{Process Rewards Should be Advantages, Not Value Functions}
\label{subsec:adv_not_value}
\vspace{-0.2cm}
To understand the relationship to potential functions, we first study test-time beam search, and present some challenges with the reward design of \citet{snell2024scaling}, that uses value function $Q^\pi(\bs, a)$ of the base policy $\pi$ 
to reward action $a$ at state $\bs$. 
Consider the example in \Figref{fig:illustration_panel}(a), where
from the $2$ states in the beam, we sample $3$ actions. If we pick next states purely based on highest values of $Q^\pi$, we would be comparing steps sampled from different states (\textit{e.g.}, $a_{1,1}$ vs. $a_{2,1}$) against each other. Clearly, a reduction in expected final outcome, \textit{i.e.},  $Q^\pi(\bs_1, a_{1,1}) - V^\pi(\bs_1)$, means that $a_{1, 1}$ \emph{by itself} has a negative effect of $-0.05$ on the probability of success from $\bs_1$, whereas $a_{2,1}$ has a positive effect of $+0.20$ from $\bs_2$. However, expanding the beam based on \emph{absolute} values of $Q^\pi$ retains the action that makes negative progress, and removes state $\bs_2$ from the beam (as beam size is 2). In other words, $Q^\pi$ fails to decouple the ``evaluation'' of an action (step), from the ``promise'' shown by the previous state. This will not be an issue for every problem, and particularly not when the beam capacity is unbounded, but under finite computational and sampling constraints, using $Q^\pi$ might retain states with potentially unfavorable steps that hurt the overall likelihood of success. \textbf{{If we could also also utilize the \emph{progress} made by the previous step}} along with the likelihood of success $Q^\pi$ when deciding what to retain in the beam, then we can address this tradeoff. 

\emph{\textbf{How can we measure the ``progress'' made by a step?}}  One approach is to consider the relative increase/decrease in the likelihood of success, before and after the step. This notion is formalized by the advantage (\Eqref{eq:advantage})  of a step under policy $\pi$. Furthermore, since advantages can attach either positive or negative values to a step, training the base policy against advantages supervises the base policy when it generates a step that makes progress (where $A^\pi > 0$), and also when it fails to produce one, employing a ``negative gradient'' that speeds up  RL training~\citep{tajwar2024preference}.
{
\setlength{\abovedisplayskip}{10pt}
\setlength{\belowdisplayskip}{10pt}
\begin{align}
    \label{eq:advantage}
    A_{}^\pi(\bs_{h}, a_{h}) \coloneqq Q_{}^\pi(\bs_{h}, a_{h}) - V_{}^\pi(\bs_{h}) = Q_{}^\pi(\bs_{h}, a_{h})  - Q_{}^\pi(\bs_{h-1}, a_{h-1}).
\end{align}
}Recall that since we view process rewards as potential functions in the MDP, they can be computed under any policy $\mu$, which can be the base policy. However, in the above example, reasons for which $Q^\pi$ is a seemingly unfit choice for process rewards also apply to $Q^\mu$. Nevertheless, we can possibly use advantage under $\mu$: $A^\mu$, which measures the progress made by a step to improve the likelihood of success under $\mu$. 
In that case, how should we choose this policy $\mu$, that we call the prover policy, and should it be necessarily different from base policy $\pi$?
Before diving into the choice of $\mu$, we discuss a more pertinent question: how should we use $A^\mu$ in conjunction with outcome rewards for improving the base policy $\pi$? We will then formally reason about the choice of $\mu$ in \TwoSecrefs{subsec:stylized_problem} {subsec:theory}.

\vspace{-0.2cm}
\subsection{Our Approach: Process Advantage Verifiers (\methodname)}
\label{subsec:method}
\vspace{-0.2cm}
For building an approach that uses process rewards $A^\mu$ together with the outcome reward $\Rex$ to improve the base policy $\pi$, we situate ourselves in the context of improving $\pi$ with online RL. If all we had was access to $\Rex$ on $\mathcal{D}$, the standard RL objective is given by:
{\setlength{\abovedisplayskip}{10pt}
\setlength{\belowdisplayskip}{10pt}
\begin{align}
    \label{eq:standard_rl}
    \ell_\mathrm{ORM-RL}(\pi) &:= \Exp_{\bx \sim \mathcal{D}, (a_1, \ldots, a_H) \sim \pi(\cdot|\bx)} \left[ \mathrm{Rex}\left( (\bx, a_1, \ldots, a_H), \by^{\star}_{\bx}\right) \right]. 
\end{align}}Inspired by how reward bonuses (and potential functions) are additive~\citep{ng1999policy,bellemare2016unifying}, one way to use process rewards $A^\mu$ is to combine it with the standard RL objective as:
{
\setlength{\abovedisplayskip}{10pt}
\setlength{\belowdisplayskip}{10pt}
\begin{align}
    \label{eq:prm_rl}
    \ell^{\pi'}_\mathrm{PAV-RL}(\pi) := \ell_\mathrm{ORM-RL}(\pi) + \alpha \cdot \textcolor{red}{\sum_{h=1}^H  
    \Exp_{\bs_h\sim d^{\pi'}_h} \Exp_{a_h \sim \pi(\cdot \mid \bs_h)} \left[ A^\mu (s_h, a_h) \right]} 
    \end{align}
}The term in red is the difference in likelihoods of success of the prover $\mu$, summed over consecutive steps (a notion of \emph{\textbf{progress}}). Here, $d^{\pi'}_h$ denotes the distribution over states at step $h$, visited by the old policy $\pi'$ (policy at previous iterate). 
Following policy gradient  derivations~\citep{williams1992reinforce}:
{
\setlength{\fboxrule}{0.75pt}
\setlength{\abovedisplayskip}{12pt}
\setlength{\belowdisplayskip}{12pt}
\begin{align}
     \label{eq:prm_rl_grad}
    \boxed{\nabla_\pi  \ell^{\pi'}_{\mathrm{PAV-RL}} (\pi) \Big\vert_{\pi' = \pi} = \sum_{h=1}^H \; \nabla_\pi  \log \pi(a_h \mid \bs_h) \cdot \underbracket{\paren{Q^\pi(\bs_h, a_h) + \alpha \cdot \textcolor{red}{A^\mu(\bs_h, a_h)}}}_{\text{effective reward}}}
\end{align}
}At a glance, we can view $Q^\pi(\bs_h, a_h) + \alpha A^\mu(\bs_h, a_h)$ as the effective reward for step $a_h$ when scored against a combination of the outcome evaluation $\Rex$, i.e., $Q^\pi$, and process rewards $A^\mu$.
Thus, we can optimize ~\Eqref{eq:prm_rl} indirectly via \textbf{(a)} running beam-search against the effective reward; or \textbf{(b)} online RL where the policy gradients are given by \Eqref{eq:prm_rl_grad}. For either of these, we need access to verifiers that are trained to predict the advantage $A^\mu(\bs_h, a_h)$ under the prover. We refer to these verifiers as \emph{\textbf{process advantage verifiers (PAVs)}}. In \Secref{subsec:dataset} we describe how to train PAVs, but now we use the above formulation to reason about how to choose prover $\mu$ that is most effective at improving base $\pi$.  

We also remark that the term in red resembles prior work on imitation learning via  policy optimization~\citep{ross2014reinforcement,sun2017deeply}, where the main aim is to learn a policy $\pi$ that imitates the prover $\mu$, or to improve upon it to some extent. Of course, this is limiting since our goal is to not just take actions that perform at a similar level as $\mu$, but to improve the base policy even further, and using a combination of $Q^\pi$ and $A^\mu$ is critical towards this goal.

\textbf{How should we choose the prover $\mu$?} Perhaps a natural starting point is to set the prover to be identical to the base policy, \textit{i.e.}, $\mu = \pi$, which produces process rewards that prior works have considered~\cite{shao2024deepseekmath}.   
However, setting $A^\pi = A^\mu$ in \Eqref{eq:prm_rl_grad} results in exactly the same policy gradient update as only optimizing outcome evaluation $\Rex$.
Moreover, for a poor base policy $\pi$, where $Q^\pi \approx 0$ on most states, the term $A^\pi$ would also be $\approx 0$, and hence running beam search with the effective rewards would not be informative at all. 
Hence, \emph{\textbf{a better approach is to use a different prover policy}}, but a very weak prover $\mu$ will likely run into similar issues as a poor base policy. We could instead use a very capable prover $\mu$, but unfortunately even this may not be any better than optimizing  only the outcome reward either. To see why, consider a scenario where $\pi$'s response contains an intermediate step that does not help make progress towards the solution (\textit{e.g.}, $\pi$ simply restates the question, see \Figref{fig:illustration_panel}(b)). Here, $Q^\mu$ for a capable prover before and after this irrelevant step will be identical since $\mu$ can succeed from either step.
This means that $\mu$ fails to distinguish steps, resulting in $A^\mu\approx 0$ in most cases. 
Training with this process reward during RL will then lead to gradients that are equivalent to those observed when purely optimizing $\ell_{\mathrm{ORM-RL}}$. 
In fact, empirically, we observe that online RL with $Q^\mu$ from strong provers leads to polices that only produce re-phrasings of the question (\Appref{sec:q_value_failure}) and do not succeed at solving the question. 
Clearly, \emph{any} policy different from the base policy cannot serve as a prover. So, how do we identify a set of good provers? Can they indeed be weaker than the base policy?  We answer next.

\begin{AIbox}{Takeaway: What should process rewards measure during test-time search and online RL?}
\begin{itemize}[leftmargin=0em]
    \setlength\itemsep{0em}
    \item Process rewards should correspond to progress, or \textbf{advantage}, as opposed to absolute $Q$-values, for a better explore-exploit tradeoff during beam search and online RL. 
    \item Advantages should be computed using a \textbf{prover} policy, different from the base policy. 
\end{itemize}
\end{AIbox}

\vspace{-0.2cm}
\subsection{Analysis in a Didactic Setting: Learning a Planted Sub-sequence}
\label{subsec:stylized_problem}
\vspace{-0.2cm}
In this section, we aim to characterize prover policies that are effective in improving the base policy. To do so, we first introduce a didactic example, representative of real reasoning scenarios to illustrate the main intuition. Then, we will formalize these intuitions in the form of theoretical results.

\textbf{Didactic example setup.}  Given an unknown sub-sequence $\by^\star$ consisting of tokens from vocabulary $\mathcal{V} \coloneqq \{1,2, \ldots, 15\}$, we train a policy $\pi$ to produce a response which contains this sub-sequence. The task completion reward is terminal and sparse, \textit{i.e.}, $r(\by, \by^\star) = 1$ for a $\by$ if and only if $\by^\star$ appears in $\by$. By design, the reward $r(\by, \by^\star)$ resembles outcome reward $\Rex(\by, \by^\star_{\bx})$ in \Secref{sec:prelim}. The prover policy $\mu$ is a procedural policy, parameterized by a scalar $\gamma > 0$ (details in \Appref{app:toy_additional}).  As $\gamma$ increases, the performance of $\mu$ improves and  $\rightarrow 1$ as $\gamma \rightarrow \infty$. For simplicity, we assume oracle access to ground-truth $A^\mu$ and $Q^\pi$, and alleviate errors from learned verifiers approximating these values. 

\textbf{\textbf{(1)} RL with effective reward $Q^\pi + \alpha A^\mu$ is $10\times$ more sample-efficient than only outcome reward.} In \Figref{fig:didactic_panel}(a), we first note that training $\pi$ with this  effective reward  under a prover $\mu$ with strength $\gamma=10$, produces optimal performance (100\% accuracy) in 350 iterations, despite starting from a mediocre initialization for $\pi$ $(\gamma=5.0)$. Training with only outcome reward is ineffective. More importantly, in \Figref{fig:didactic_panel}(b), we note that effective rewards only help for a set of provers, in $\gamma \in [8.0, 15.0]$. Outside this range, we observed advantages $A^\mu$ were close to $0$ on most states, either because $\mu$ was poor (small $\gamma$) and was unable to generate $\by^\star$ even when $\pi$ got the sequence partially correct, or because $\mu$ was strong (large $\gamma$) that it generated $\by^*$ with almost equal likelihood from all prefixes. 

\textbf{\textbf{(2)} Effective reward improves Pass @N by $5\times$ over only outcome reward.}
 We report the ``Pass @N'' performance in \Figref{fig:didactic_panel}(c), which measures the maximum reward $r$ across $N$ traces sampled \textit{i.i.d.} from $\pi$ and hence, represents the ceiling on the performance of any test-time search method that picks a single response from multiple draws (\textit{e.g.}, as in Best-of-N).
For a policy trained  with the effective reward for 100 iterations, the Pass @N performance grows $5\times$ faster with $N$, compared to the policy trained with only the outcome reward. Due to only sparse feedback, the latter policy does not learn to sample partially correct $\by^\star$, whereas a policy trained with the effective reward produces partially correct $\by^\star$, and is able to sample the complete $\by^\star$ with higher likelihood during Pass @N.

\begin{figure}
    \centering
    \includegraphics[width=0.95\linewidth]{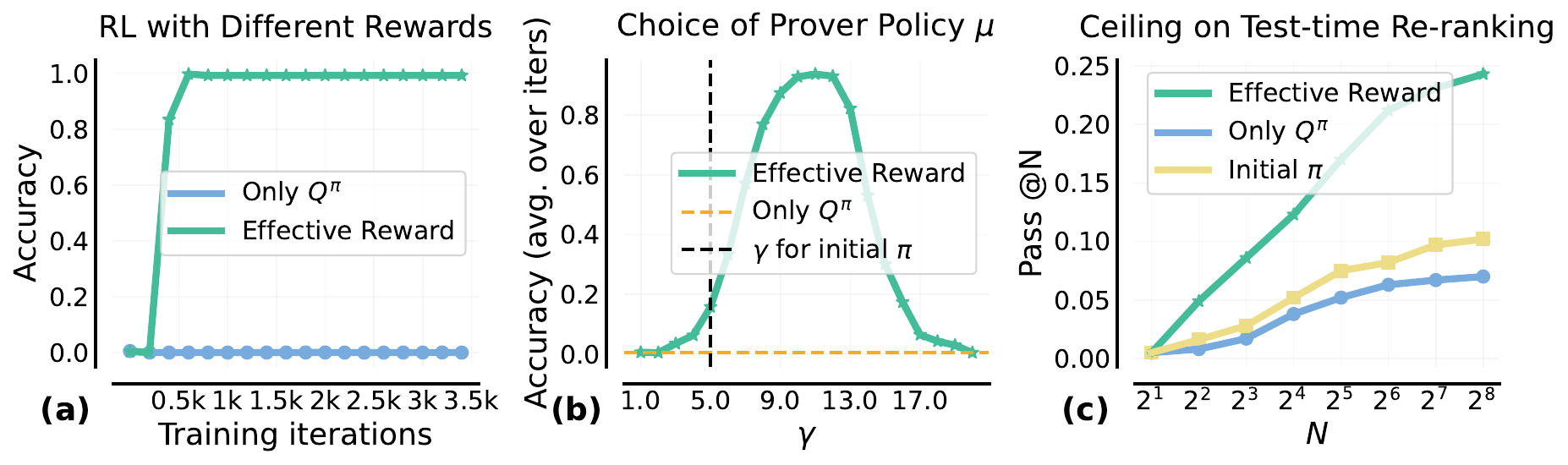}
    \vspace{-0.3cm}
    \caption{\textbf{\emph{Results for our didactic analysis:}} \textbf{(a):} We train base policy via RL with either effective reward $Q^\pi + \alpha A^\mu$, or the typical $Q^\pi$ (computed via Monte-Carlo sampling).  \textbf{(b):} We vary the strength $\gamma$ of the prover $\mu$ used to compute advantages $ A^\mu$ in the effective reward, and plot the base policy accuracy  averaged over the RL run.  \textbf{(c):} We plot the max score out of $N$ responses  (Pass @N)  sampled \textit{i.i.d.} from an undertrained base policy (iter 100) .}
    \vspace{-0.3cm}
    \label{fig:didactic_panel}
\end{figure}

\begin{AIbox}{Takeaway: Online RL with process rewards from different prover policies.}
Effective rewards $Q^\pi + \alpha A^\mu$ from prover $\mu$: (i) improve sample efficiency of online RL, and (ii) yield policies with better Pass @N performance, over using only outcome rewards. But, advantages of very capable or poor $\mu$ do not improve base policy beyond outcome rewards. 
\end{AIbox}

\vspace{-0.2cm}
\subsection{Theory: Provers Complementary to the Base Policy Boost Improvement}
\label{subsec:theory}
\vspace{-0.2cm}

From our didactic analysis, it is clear that process rewards $A^\mu$ under different provers $\mu$ disparately affect the base policy that optimizes $Q^\pi + \alpha A^\mu$ via online RL. We now present a formal analysis of why this happens and characterize a class of provers that can guarantee non-trivial improvements to the base policy.
For simplicity, we assume oracle access to $Q^\pi, A^\mu$ at every state-action pair $(\bs_h, a_{h})$ and prove our result in the tabular RL setting, where the policy class is parameterized using the softmax parameterization in \citet{agarwal2021theory}. Proofs for this section are in \Appref{app:theory-convergence}.

\textbf{Main intuitions.} We expect a prover $\mu$ to improve a base policy $\pi$ only when \textbf{$\mathbf{\mu}$ is able to  \emph{distinguish} different actions taken by $\pi$}, by attaining sufficiently varying advantage values $A^\mu(\bs_h, a)$ for actions $a$ at state $\bs_h$. 
This can be formalized under the notion of sufficiently large variance across actions, $\mathbb{V}_{a\sim \pi} \left[A^\mu(\bs_h, a)\right]$. In that case, can we simply use a policy with large advantage variance under any measure?
No,  because when the prover $\mu$ ranks actions at a given state very differently compared to the base policy $\pi$ (e.g., if $A^\mu$ and $A^\pi$ are opposite), then effective rewards  $Q^\mu + \alpha A^\pi$ will be less reliable due to conflicting learning signals. Thus, we want $\Exp_{\pi}\left[\langle A^\mu, A^\pi \rangle \right]$ to not be too negative, so that \textbf{$\mathbf{\mu}$ and $\mathbf{\pi}$ are reasonably \emph{aligned}} on their assessment of steps from $\pi$. 

In Theorem~\ref{thm:policy-improvement}, we present our result on policy improvement where the base policy is updated with natural policy gradient~\citep{Kakade2001}: $\pi_{t+1}(a\mid \bs_h) \propto \exp(\gamma \cdot (Q^\pi(\bs_h, a) + A^\mu (\bs_h, a)))$. 
We note that in this idealized update rule, swapping $Q$ values (of $\mu$ or $\pi$) with advantages does not affect the update since we assume access to all possible actions when running the update. Nonetheless, despite this simplifying assumption, the analysis is able to uncover good choices for the prover policy $\mu$ for computing process reward $A^\mu$, and is orthogonal to the design consideration of advantages or $Q$-values as process rewards that we have discussed so far in this paper.
Theorem~\ref{thm:policy-improvement} formalizes our intuition by showing that policy improvement at iteration $t$, grows as the variance in $A^\mu$ values increases (higher distinguishability) and reduces when $A^\mu$ and $A^\pi$ become extremely misaligned. This will then allow us to discuss a special case for the case of Best-of-K policies as provers as an immediate corollary.

\begin{theorem}[Lower bound on policy improvement; informal]
\label{thm:policy-improvement}
For base policy iterate $\pi_t$, after one step of policy update, with learning rate $\gamma \ll 1$, the improvement over a distribution of states $\rho$:  
{
\setlength{\abovedisplayskip}{10pt}
\setlength{\abovedisplayskip}{10pt}
\begin{align}
    \Exp_{\bs\sim \rho} \brck{V^{\pi_{t+1}}(\bs)- V^{\pi_{t}}(\bs)} \gsim  \gamma \cdot \underbracket{ \Exp_{\bs\sim \rho} \mathbb{V}_{a\sim \pi_{t}}[A^\mu(\bs, a)]}_{\text{distinguishability from $\mu$}} +  \gamma \cdot \underbracket{ \Exp_{\bs\sim \rho} \Exp_{a\sim\pi_t} \brck{A^\mu(\bs, a)A^{\pi_t}(\bs, a)}}_{\text{alignment between $\pi_t$ and $\mu$}}
\end{align}
}
\end{theorem}

It may seem that the base policy $\pi$ can only learn from an improved prover $\mu$, but our result
 shows that a \textbf{weak prover can also amplify a stronger base policy}, since a weak prover $\mu$ may have a lower average of $Q^\mu$ under its own measure, but still have higher variance across $Q^\mu$ (compared to $Q^\pi$) when evaluated under $\pi$ (see Proposition~\ref{prp:learning-signal} in \Appref{subsec:prp-learning-signal-proof} for formal discussion).
 This tells us that \emph{\textbf{rewarding progress under a prover is different from typical knowledge distillation or imitation learning algorithms}}~\citep{hinton2015distilling,rusu2015policy} that in most cases remain upper bounded by the performance of the stronger teacher. So provers cannot be characterized purely by strength, what is a class of provers that is a reasonable starting point if we were to improve any base policy $\pi$?
 
 \textbf{The policy class of ``Best-of-K'' (computed over base policies) contain complementary provers.} A good starting point to identify good provers for a base policy $\pi$, is the class of Best-of-K policies or $\mathrm{BoK}(\pi)$. Recall from \Secref{sec:prelim} that the performance of $\mathrm{BoK}(\pi)$ increases monotonically with $K$. Applying Theorem~\ref{thm:policy-improvement} to this class, we arrive at Remark~\ref{rem:bok-result} that recommends using $\mathrm{BoK}(\pi)$ with $K>1$ as a prover policy for a poor base policy $\pi$.  However, $K$ cannot be too large always since when $Q^\pi(\bs, a) \approx 1$ , increasing $K$ too much can hurt distinguishability of different steps at that state. In the next section, we empirically note that the policies in the class of $\mathrm{BoK}(\pi)$ indeed induce different performance gains when used as prover policies, and we find $\mathrm{Bo4}$ to be a good choice for test-time search over most base policies.  
 \begin{remark}
\label{rem:bok-result} When $Q^\pi(\bs, a) = O(1/K), \forall \bs, a$, using $\mathrm{BoK}(\pi)$ as a prover for base $\pi$ improves distinguishability (and improvement) by $\Omega(K^2)$, and make  alignment worse at most by $O(K)$.
\end{remark}
\begin{AIbox}{Takeaway: Formal characterization of good prover policies that improve the base policy.}
Provers with advantages that can \textbf{distinguish} actions taken by the base policy (more strongly than the base policy itself) but are \textbf{not too misaligned} from the base, boost improvements on each update of the base policy. We call such policies \textbf{\emph{complementary provers}}. $\mathrm{BoK}(\pi)$ for any base policy $\pi$ for $K>1$ can provide a good starting choice of prover policies.  
\end{AIbox}

\vspace{-0.2cm}
\section{Results: Scaling Test-Time Compute with PAVs}
\label{sec:prms_for_bon}
\vspace{-0.2cm}

\definecolor{lapislazuli}{rgb}{0.15, 0.38, 0.61}

\begin{figure}[!t]
    \centering
    \includegraphics[width=0.95\linewidth]{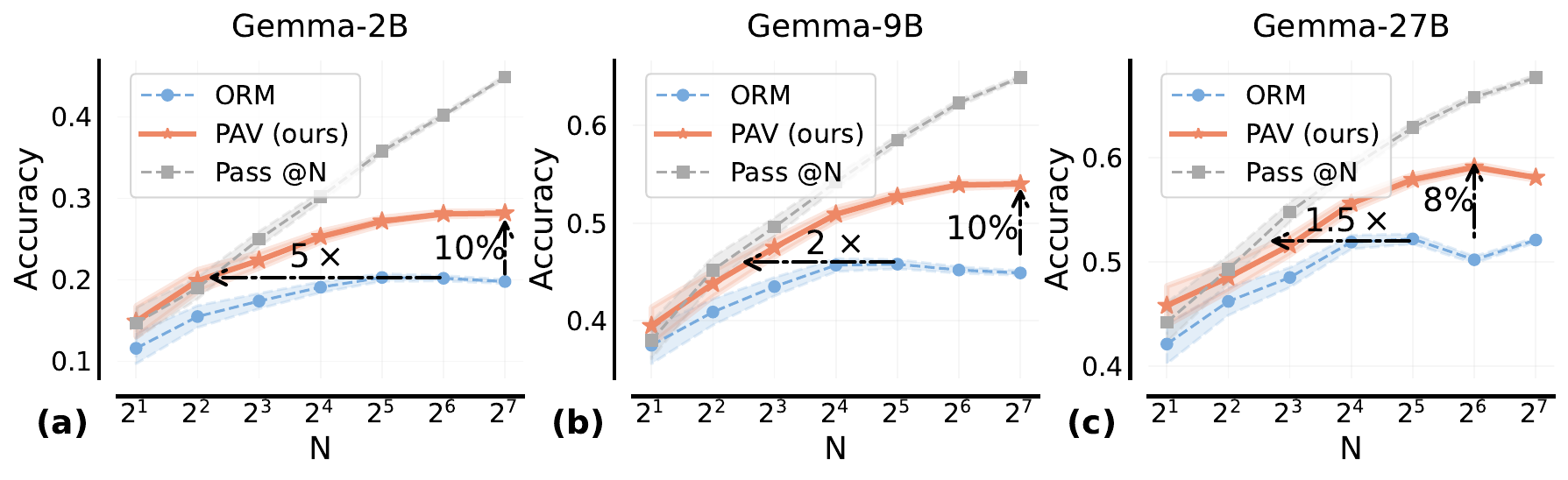}
    \vspace{-0.3cm}
    \caption{\textbf{\textit{For test-time search, PAVs are $8-10\%$ more accurate and $1.5-5 \times$ more compute efficient over ORMs:}}  On samples from \textbf{(a)} Gemma-2B ,  \textbf{(b)} 9B , and \textbf{(c)} 27B SFT policies, we run test-time beam search with the estimate of effective reward $Q^\pi + \alpha A^\mu$ (PAV), where $\mu$ is the $\mathrm{Bo4}(\pi)$ policy. We compare beam search  performance with best-of-N, re-ranking with a trained outcome verifier (ORM), or the oracle $\Rex$ (Pass @N).
    }
    \vspace{-0.45cm}
    \label{fig:bon_panel_1}
\end{figure}

Now, we study how process verifiers can scale up test-time compute. While our derivations from \Secref{subsec:method} were with RL, we can also use the \emph{effective reward} $Q^\pi(\bs_h, a_h) + \alpha \cdot {A^\mu(\bs_h, a_h)}$ for running beam search over intermediate steps sampled from base policy $\pi$. To do so, we train a process advantage verifier to predict $A^\mu$, along with a process reward model $Q^\pi$. PAV training is done using procedures discussed in \Secref{subsec:dataset}. While the candidates of the beam are selected using a combination of both the PAV and the PRM $Q^\pi$, the final candidate is selected using the outcome reward prediction from $Q^\pi$ itself (i.e., we repurpose the PRM representing $Q^\pi$ as an ORM). 
For clarity, we abuse notation and refer to the estimated effective reward (ORM + $\alpha$ PAV) as PAV directly.

\textbf{Setup}. We finetune Gemma 2B, 9B, and 27B~\citep{team2024gemma} on  MATH~\citep{hendrycksmath2021} via supervised fine-tuning (SFT) to get three base policies. 
The set of provers consists of the three base SFT policies  themselves as well as their best-of-K policies for different values of $K \in \{2^{0},\ldots,2^{5}\}$.
Additional details for the experiments in this section are in  \Appref{app:test_time_scaling_additional}.

\vspace{-0.2cm}
\subsection{PAVs Scale Test-Time Compute by \texorpdfstring{$5 - 10 \times$}{Xx} Over ORMs}
\vspace{-0.1cm}
\label{subsec:pavs-scale-test-time-compute}

\textbf{\textcolor{lapislazuli}{Result 1: PAVs are more compute efficient than ORMs.}} In \Figref{fig:bon_panel_1}, we plot the performance of beam search with PAVs for different sizes of the beam $N$, and compare it with  best-of-$N$ using ORMs, \textit{i.e.}, sampling $N$ complete solutions from the base policy and returning the one with the highest ORM score. To compare PAVs and ORMs, we evaluate the \emph{compute efficiency} of PAVs over ORMs, given by the ratio of total compute needed by PAVs to obtain the same performance as running best-of-128 with ORM. Even when accounting for the fact that running beam search with PAVs does require additional compute \emph{per solution trace} (since each element in the beam samples $C=3$ next steps, before scoring and pruning the beam), PAVs are able to scale the compute efficiency by \mathbf{$10\times$} over ORMs for Gemma-2B, 9B base models, and by $5\times$ for Gemma-27B model. 
We use $\mathrm{BoK}(\pi)$ with $K=4$ as the prover policy for all base policies $\pi$. 

We also compare performance with beam search using process verifiers that only predict $Q^\pi$, and best-of-N where the ORM is replaced with PAV (PAV-as-ORM). At $N=128$, similar to \citet{luo2024improve}, we note a similar gain of $4\%$ for ``PAV-as-ORM''~\Figref{fig:bon_panel_2}(a) 
over only ORMs, for base Gemma-9B $\pi$. When comparing beam search with $Q^\pi$~\citep{snell2024scaling}, we find that PAVs scale compute efficiency by $8\times$. Evidently, advantages from the prover in the effective reward positively impact the beam search. Why does $A^\mu$ help, and for what choice of the prover $\mu$?

\begin{figure}[!t]
    \centering
    \includegraphics[width=0.99\linewidth]{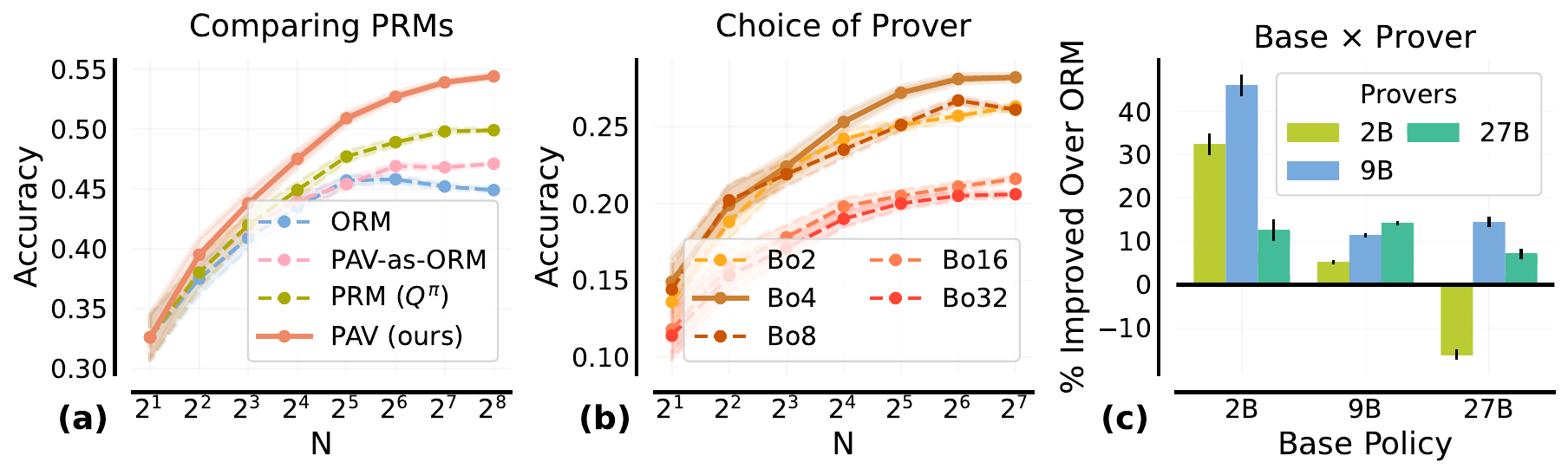}
    \vspace{-0.3cm}
    \caption{
    \textbf{\textit{Comparing PAVs with search baselines and ablating over the prover policy:}} \textbf{(a):} We compare beam search over Gemma 9B SFT, using either effective reward (PAV), or $Q^\pi$~\citep{snell2024scaling}, and report best-of-N performance where the re-ranker is either the ORM or PAV-as-ORM. \textbf{(b):} For the base Gemma 2B SFT policy, we run beam search with the effective reward where the prover is $\mathrm{BoK}(\pi)$ for different values of $K$. In both $(a), (b)$ the x-axis scales the size of the beam or $N$ for best-of-N.  \textbf{(c):} For each base policy in the set: Gemma 2B, 9B, 27B policies, we run beam search with PAVs (beam size of 16) where the prover is another policy from the same set.} 
    \label{fig:bon_panel_2}
\end{figure}

\textbf{\textcolor{lapislazuli}{Result 2: Beam search with too weak/strong provers is sub-optimal.}} 
In \Figref{fig:bon_panel_2}(b), for the setting when the base policy $\pi$ is a Gemma-2B SFT model, we compare beam search with PAVs where the provers are given by $\mathrm{BoK}(\pi)$, for different values of $K$. Recall that as $K$ increases, $\mathrm{BoK}(\pi)$ becomes stronger.
 Corroborating our analysis in \Secref{subsec:theory}, our results show that neither too weak ($\mathrm{Bo}2$) or too strong ($\mathrm{Bo}32$) provers perform best. Instead, across all values of $N$, we find $\mathrm{Bo}4$ to be dominant. The advantage values $A^\mu \approx 0$ on all steps for very large $K$, since $Q^\mu(\bs_h, a_h) = 1-(1-Q^\pi(\bs_h, a_h))^K \rightarrow 1$ on all steps, as we increase $K$. 
 Hence, \emph{\textbf{in order to succeed we need an intermediate-level prover policy}}.

We make similar observations in Figure~\ref{fig:bon_panel_2}(c) where we use the three base policies (Gemma 2B/9B/27B) 
as provers for training PAVs. In this scenario, we evaluate beam search with PAVs at $N=16$ on top of different base policies. We find that for the 2B and 9B base models, the 9B and 27B provers are most effective respectively, whereas for the 27B model, \textbf{\emph{surprisingly a weaker 9B policy is more effective than the stronger 27B model.}} The weaker model presumably offers a complementary signal that distinguishes between different actions taken by 27B, aligning with our theoretical observations in \Secref{subsec:theory}. 

\begin{figure}[!b]
    \centering
    \includegraphics[width=0.7\linewidth]{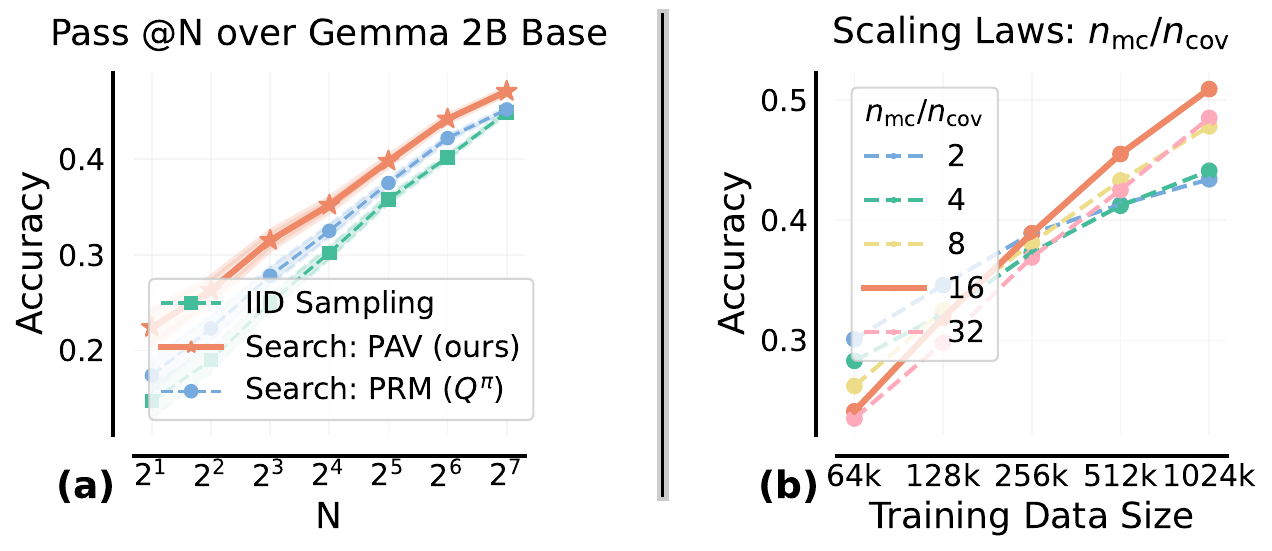}
    \vspace{-0.3cm}
    \caption
    {
    \textbf{(a):} Beam search with PAVs improves exploration efficiency (higher Pass@N), over typical PRMs.
    \textbf{(b):} Performance of beam search over Gemma 9B SFT for PAVs trained on datasets with different  $n_{\mathrm{mc}}/n_{\mathrm{cov}}$.
    }
    \label{fig:scaling_panel_2}
\end{figure}

\textbf{\textcolor{lapislazuli}{Result 3: Advantages from the prover policy enable exploration.}} As discussed in \Secref{subsec:adv_not_value},  advantage $A^\mu$ measures the progress made by an action agnostic of the value of the previous state, where as $Q^\pi$ measures the promise of a particular state. Given a finite capacity beam, our effective reward (\Eqref{eq:prm_rl_grad}), which linearly combines \emph{$Q^\pi$ and $A^\mu$ induces a better tradeoff between exploring new prefixes (states) from where progress can be made and exploiting currently known prefixes with high Q-values.}  
Exploration at initial steps is critical to ensure that the beam at later steps covers diverse partial rollouts each with a high likelihood of producing the correct answer. Thus over-committing to the beam with actions from the same state, regardless of the progress made by each can prove to be sub-optimal over a selection strategy that balances rewarding previous actions $A^\mu$ and current states $Q^\pi$. Indeed, we observe in \Figref{fig:scaling_panel_2}(a), beam search with PAV enhances pass@N performance vs. beam search with $Q^\pi$ and \textit{i.i.d.} sampling.

\begin{AIbox}{Takeaways: Scaling test-time compute with process advantage verifiers.}
\begin{itemize}[leftmargin=0em]
    \setlength\itemsep{0em}
    \item Beam search with PAVs boosts accuracy by >8\% \& compute efficiency by 1.5-5x over ORMs.  
    \item Utilizing Best-of-K policies (corresponding to the base policy) as provers induce better exploration to maximize outcome reward. Optimal provers for a base policy appear at $K>1$.
\end{itemize}
\end{AIbox}

\vspace{-0.2cm}
\subsection{How to Collect Data to Train PAVs?: PAV Training Data Scaling Laws}
\label{subsec:dataset}
\vspace{-0.2cm}

We now describe the procedure for training outcome verifiers and PAVs. We can learn to predict $Q^\pi$ for a policy $\pi$ (similar for $Q^\mu$) by finetuning LLMs with a cross-entropy loss on the following data with triplets $(\bs, a, Q_\mathrm{mc}^\pi(\bs, a))$. To collect this data, we first sample $n_\mathrm{cov}$ {``seed''} rollouts from the base or prover policy respectively for ORM and PAVs, to promote coverage over prefixes and steps.  
Then we sample $n_\mathrm{mc}$ additional rollouts, conditioned on each prefix in the seed rollout to compute the Monte-Carlo estimate of $Q^\pi$ at each prefix. 
In \Figref{fig:scaling_panel_2}(b) we plot the beam search performance  of PAVs trained with different ratios of $\nicefrac{n_\mathrm{mc}}{n_\mathrm{cov}}$, as we scale the total dataset size. Here, the beam size is fixed to 128 and the base policy is the Gemma 9B SFT policy and prover is $Bo4$ policy.
\textbf{We find that} under low sampling budgets, optimizing for coverage (${n_\mathrm{cov}} > {n_\mathrm{mc}}$) is better for performance, and when budget is higher, reducing label noise in $Q^\pi_\mathrm{mc}$ by setting ${n_\mathrm{mc}} > {n_\mathrm{cov}}$ gets us more improvements. 
In addition, we also spend some initial sampling budget is spent to identify ``high value'' states where $Q^\pi$ is larger than a threshold, and identify the first step with low $Q^\pi$ on an incorrect partial rollout from this state. We found this strategy to scale better with dataset size, as we discuss in \Appref{app:dataset_pavs}.

\vspace{-0.2cm}
\section{Results: Scaling Dense-Reward RL with PAVs}
\label{sec:prms_for_RL}
\vspace{-0.2cm}

We can also use PAVs to train policies via online reinforcement learning (RL), by using the  effective reward $Q^\pi + \alpha A^\mu$ as dense, per-step rewards. We compare the sample efficiency of PAV-RL (i.e., $\ell_{\mathrm{PAV-RL}}$ in \Eqref{eq:prm_rl}) with standard ORM-RL (i.e., $\ell_{\mathrm{ORM-RL}}$ in \Eqref{eq:standard_rl}) on Gemma 2B and Gemma 9B SFT models, which are further optimized via rejection finetuning (RFT)~\citep{yuan2023scaling}, before using them to initialize RL. To our knowledge, no prior work has successfully demonstrated the use of dense per-step feedback with a process reward model for RL, and we present the first  significant set of  results establishing the efficacy of this approach. We show that PAV-RL is much more sample-efficient, and enjoys a higher ceiling on the performance of any test-time re-ranker. Additional details for the experiments are in \Appref{app:rl_with_pavs_additional}.

\begin{figure}[!t]
    \centering
    \includegraphics[width=0.99\linewidth]{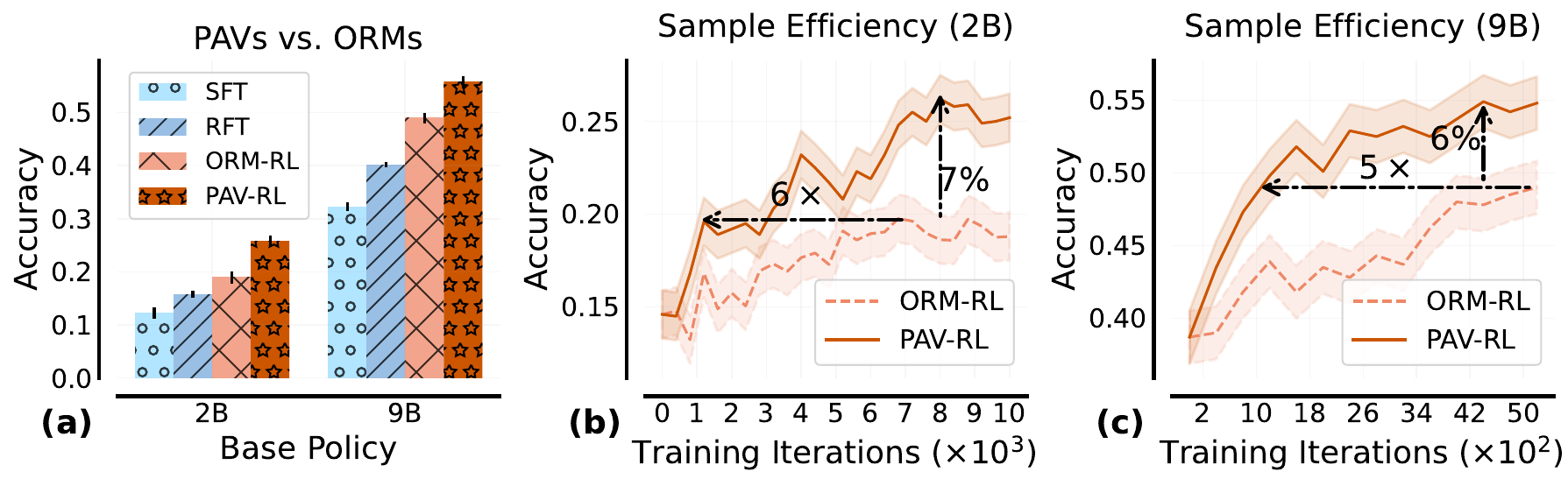}
    \caption{\textbf{\textit{PAVs as dense rewards in RL improve sample efficiency compared to ORMs, along with gains on raw accuracy:}}
    \textbf{(a)} We report the performance of a base policy trained using RL with effective rewards (PAV-RL), or only outcome rewards (ORM-RL), and baselines SFT, RFT.
    \textbf{(b,c):} Across training iterations, we report the test performance of policies trained with PAV-RL and ORM-RL, on Gemma 2B and 9B SFT base policies.
    }
    \label{fig:rl_panel}
\end{figure}

\textbf{\textcolor{lapislazuli}{Result 1: PAV-RL is $> 7\%$ better than ORM-RL in test accuracy, and $6 \times$ sample efficient.}}
In \Figref{fig:rl_panel}(a), we report the test accuracies of Gemma 2B and 9B models trained with SFT, RFT, ORM-RL and PAV-RL. PAV-RL improves the RFT policy by $11\%$ for 2B, and $15\%$ for 9B, with  $>7\%$ gain over ORM-RL in both cases. Not only do the effective rewards from PAV improve the raw accuracy after RL, this higher accuracy is attained $6\times$ faster (see \Figref{fig:rl_panel}(b)) for the 2B run and similarly for the 9B RL run (\Figref{fig:rl_panel}(c)). For both 2B and 9B, RL runs, we experiment with two options for the prover policy: \textbf{(i)} 2B SFT policy; and \textbf{(ii)} 9B SFT policy. While both of these provers rapidly become weaker than the base policy within a few gradient steps of RL, a fixed PAV trained with each of these provers is able to still sustain performance gains in RL. More interestingly, we find that the 2B SFT policy serves as the best choice of the prover for both 2B and 9B policies. This observation that a weak prover can still improve the base policy corroborates our results in the didactic setup and our analysis in \Secref{subsec:theory}. While we were not able to run experiments where the prover policy is dynamically updated on the fly, we believe that updating the prover through the process of RL training should only amplify these benefits.

\textbf{\textcolor{lapislazuli}{Result 2: 
PAV-RL achieves higher performance ceiling on test-time re-ranking.}} In \Figref{fig:rl_ablation}(a), for Gemma 2B, we plot the Pass @N performance for each method, and find \textbf{(i)}  Pass @N  is  higher ($>7\%$) for PAV-RL, compared to ORM-RL, for any $N\leq 128$; and \textbf{(ii)} the rate at which Pass @N improves for PAV-RL is higher than ORM-RL. Both trends are consistent with our observations on the didactic example in \Secref{subsec:stylized_problem}.
Notably, for $N\geq 64$, ORM-RL is worse than the SFT policy, perhaps due to lower entropy over the distribution at the next step resulting in non-diverse candidates.
Why does PAV-RL produce diverse candidates, and does not suffer from the low diversity problem in ORM-RL? We answer this with a key insight on how \emph{\textbf{the primary benefit of  PAVs is to promote efficient exploration.}}

\begin{figure}[h]
    \centering
    \includegraphics[width=0.7\linewidth]{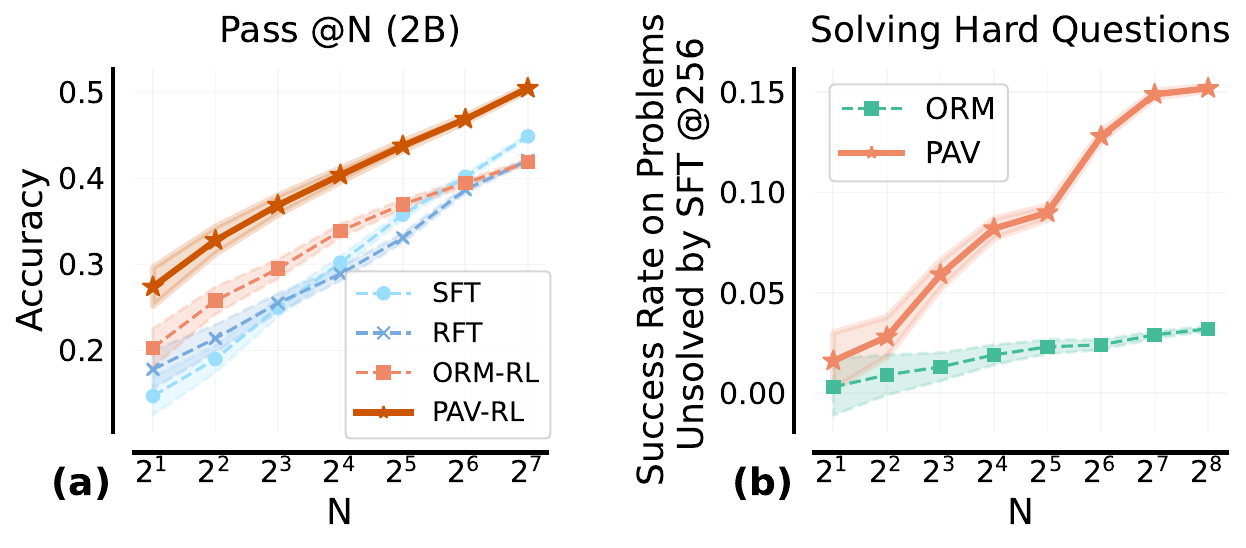}
    \caption{
    \textbf{(a):} For the policies trained in (a) we report the best-of-N performance where the oracle reward $\Rex$ is used to rank $N$ candidates sampled from the base policy (Pass @N).
    \textbf{(b):} Amongst hard problems that remain unsolved by Best-of-256 over the base SFT policy, we check how many are solved by Best-of-N over PAV-RL or ORM-RL. PAV-RL is able to solve a substantially more problems than what ORM-RL was able to solve.
    }
    \label{fig:rl_ablation}
\end{figure}

\textbf{\textcolor{lapislazuli}{Result 3: PAVs improve exploration and discover correct solutions to novel problems.}} An outcome reward model rewards downweight all steps in an incorrect rollout equally during RL, 
whereas the effective reward $Q^\pi + \alpha A^\mu$ in PAVs, up-weights steps that make progress under the prover, even when the complete rollout is incorrect. 
This increases the coverage over individual steps that can improve the likelihood of the base policy to succeed (since the prover policy is not too misaligned with the base policy). These can now be proposed by the base policy at a given prefix. This mechanism for exploration is analogous to test-time search we discussed in \Secref{subsec:pavs-scale-test-time-compute}.
Hence, the directed supervision from PAVs improves sample-efficiency throughout the course of training (\Figref{fig:rl_panel}(c)). In fact, we also find that combining the PAV-RL policy with test-time beam search is able to solve a substantially larger number of \emph{new} problems within smaller compute budgets ($N=16,32$) that the SFT policy cannot solve with a much larger budget $N=256$ (\Figref{fig:rl_ablation}(b)).

\begin{AIbox}{Takeaway: RL with process advantage verifiers (PAVs) as dense rewards}
\begin{itemize}[leftmargin=0em]
    \setlength\itemsep{0em}
    \item Using trained PAVs as dense rewards in RL boosts scales sample efficiency by $5-6\times$, compared to only using sparse ORM rewards, and results in policies with a higher Pass @N performance. 
    \item Advantages from a complementary prover policy improves the sample efficiency of exploration in RL, and produces policies that can discover solutions to hard novel questions.
\end{itemize}
\end{AIbox}

\vspace{-0.2cm}
\section{Related Work}
\label{sec:relwork}
\vspace{-0.2cm}
We briefly discuss some key related works here, and leave the detailed discussion for \Appref{app:relwork_additional}. To address issues of sparse feedback in ORMs~\citep{cobbe2021training}, recent works~\citep{lightman2023let,uesato2022solving} trained
process reward models (PRMs) to densely predict incorrect steps in a multi-step reasoning trace. Since human data collection for process labels is not scalable enough, recent work~\citep{wang2024mathshepherd,luo2024improve} used automated supervision to annotate steps with $Q$ values under the base policy, \textit{i.e.}, the PRMs score a step with the likelihood of future success, when continuing to sample from the step. While $Q$-value PRMs in \citet{lightman2023let,luo2024improve} were mainly used as verifiers for re-ranking, \citet{snell2024scaling} used them for test-time beam search.  \citet{shao2024deepseekmath} uses PRMs for RL but found a gain of only $1-2\%$ with PRMs.
In our work, we question solely relying on $Q$-values or advantages of the base policy, and find that measuring progress (i.e., advantages) under a different prover policy can amplify exploration, thus boosting test-time search and RL. To our knowledge, we are the first to show substantial gains in compute and sample efficiency with PRMs.
Our methodology for data collection is similar to~\citep{setlur2024rl,hwang2024self} (i.e., identify ``first pits'' in reasoning traces), these works only use it to collect preference pairs. Beyond all of these, we also characterize which policy to use for computing advantages.

Concurrently to us, akin to the methodology in~\citet{setlur2024rl,hwang2024self}, \citet{kazemnejad2024vineppo} optimize the base policy $\pi$ with online RL, where the dense step-level rewards correspond to advantages $A^\pi$ under the base policy $\pi$ itself. This is a special case of our setting, where the prover policy $\mu = \pi$, but as we note in \Secref{subsec:method}, setting $\mu=\pi$ in our effective reward (\Eqref{eq:prm_rl_grad}) results in exactly the same policy gradient updates as only optimizing the outcome reward. Since \citet{kazemnejad2024vineppo} use ``on-the-fly'' Monte-Carlo rollout estimation to estimate advantages, they are able to avoid estimation errors in the process reward model. Nonetheless, our theoretical result and the didactic example (both of which assume access to perfect advantage estimates) show that gains from this approach are significantly smaller than using an appropriate prover policy, which is distinct from the base policy.

\vspace{-0.2cm}
\section{Discussion and Conclusion}
\label{sec:conclusion}
\vspace{-0.2cm}

We began our exposition with the following question: how to define process rewards such that optimizing the base policy against process rewards ultimately improves the outcome level correctness of the final answer? 
Our key finding is that process rewards defined as advantages of a prover policy, distinct from the base policy improve the efficiency of exploration for steps sampled from the base policy during test-time search and online RL. This improved exploration in turn leads to discovery of better solutions, resulting in a higher accuracy on the math reasoning task.
We also formally characterized the set of good prover policies as policies with step-level advantages that meaningfully contrast steps generated by the base policy, while still producing step-level advantages that are aligned with the base policy. Having trained process advantage verifiers (PAVs) to predict advantages under the prover policy, we empirically observed that test-time search against the trained PAVs improve the compute-efficiency of search by $1.5-5\times$, and accuracy of search by over $8\%$ compared to running best-of-$N$ against an ORM. Next, we present one of the  significant results that validate the use of dense supervision when optimizing the base policy with online RL. Specifically, we show that dense online RL with rewards from our trained PAVs, improves sample efficiency of online RL by $5-6 \times$, and results in an accuracy gain of over $6\%$.

\textbf{Limitations.} Despite the promise of our results, there are several limitations to our work that present important avenues for future research. First, while we can easily compute the right hand side of our result in Theorem~\ref{thm:policy-improvement} to understand whether a given prover policy will improve a fixed base policy, it is unclear how to automatically design a flexible class of optimal (or very good) prover policies for a sequence of base policy iterates. Perhaps simultaneously optimizing the prover and the base policy (in a two-player game) might provide for an approach to obtain the best prover during RL, but this is largely an open question. Second, since inevitably learning a process advantage verifier (PAV) will incur fitting errors and this upper bounds peformance of our method. Fitting errors can be circumvented if our approach if we can simply run rollouts from prover policies during online RL or search to estimate advantages without training verifiers,  and extending our approach to this setup is a good avenue for future work.

\vspace{-0.2cm}
\section*{Acknowledgements}
\vspace{-0.2cm}
The authors would like to thank Charlie Snell, Yi Su, Katherine Heller, and Virginia Smith for feedback on an earlier version of this paper. We also thank Ahmad Beirami, Sergey Levine, Victor Veitch, Idan Shenfeld, Arian Hosseini, Stephen Pfohl, Xiangyu Qi, Tianhe Yu, and Christina Baek for technical discussions. AS and CN also thank Preston Robinette, Sho Kannan, Tianze Shi, Diana Mincu, Hritik Bansal, and Liangchen Luo for code, infrastructure and data analytics support.

\bibliographystyle{abbrvnat}
\nobibliography*
\bibliography{main}

\begin{thebibliography}{46}
\providecommand{\natexlab}[1]{#1}
\providecommand{\url}[1]{\texttt{#1}}
\expandafter\ifx\csname urlstyle\endcsname\relax
  \providecommand{\doi}[1]{doi: #1}\else
  \providecommand{\doi}{doi: \begingroup \urlstyle{rm}\Url}\fi

\bibitem[Agarwal et~al.(2021)Agarwal, Kakade, Lee, and
  Mahajan]{agarwal2021theory}
A.~Agarwal, S.~M. Kakade, J.~D. Lee, and G.~Mahajan.
\newblock On the theory of policy gradient methods: Optimality, approximation,
  and distribution shift.
\newblock \emph{Journal of Machine Learning Research}, 22\penalty0
  (98):\penalty0 1--76, 2021.

\bibitem[Anthony et~al.(2017)Anthony, Tian, and Barber]{anthony2017thinking}
T.~Anthony, Z.~Tian, and D.~Barber.
\newblock Thinking fast and slow with deep learning and tree search.
\newblock \emph{Advances in neural information processing systems}, 30, 2017.

\bibitem[Bellemare et~al.(2016)Bellemare, Srinivasan, Ostrovski, Schaul,
  Saxton, and Munos]{bellemare2016unifying}
M.~Bellemare, S.~Srinivasan, G.~Ostrovski, T.~Schaul, D.~Saxton, and R.~Munos.
\newblock Unifying count-based exploration and intrinsic motivation.
\newblock In \emph{Advances in Neural Information Processing Systems}, pages
  1471--1479, 2016.

\bibitem[Bi et~al.(2024)Bi, Chen, Chen, Chen, Dai, Deng, Ding, Dong, Du, Fu,
  et~al.]{bi2024deepseek}
X.~Bi, D.~Chen, G.~Chen, S.~Chen, D.~Dai, C.~Deng, H.~Ding, K.~Dong, Q.~Du,
  Z.~Fu, et~al.
\newblock Deepseek llm: Scaling open-source language models with longtermism.
\newblock \emph{arXiv preprint arXiv:2401.02954}, 2024.

\bibitem[Chang et~al.(2023)Chang, Brantley, Ramamurthy, Misra, and
  Sun]{chang2023learning}
J.~D. Chang, K.~Brantley, R.~Ramamurthy, D.~Misra, and W.~Sun.
\newblock Learning to generate better than your llm.
\newblock \emph{arXiv preprint arXiv:2306.11816}, 2023.

\bibitem[Chang et~al.(2015)Chang, Krishnamurthy, Agarwal, Daum{\'e}~III, and
  Langford]{chang2015learning}
K.-W. Chang, A.~Krishnamurthy, A.~Agarwal, H.~Daum{\'e}~III, and J.~Langford.
\newblock Learning to search better than your teacher.
\newblock In \emph{International Conference on Machine Learning}, pages
  2058--2066. PMLR, 2015.

\bibitem[Cobbe et~al.(2021{\natexlab{a}})Cobbe, Kosaraju, Bavarian, Chen, Jun,
  Kaiser, Plappert, Tworek, Hilton, Nakano, Hesse, and
  Schulman]{cobbe2021gsm8k}
K.~Cobbe, V.~Kosaraju, M.~Bavarian, M.~Chen, H.~Jun, L.~Kaiser, M.~Plappert,
  J.~Tworek, J.~Hilton, R.~Nakano, C.~Hesse, and J.~Schulman.
\newblock Training verifiers to solve math word problems.
\newblock \emph{arXiv preprint arXiv:2110.14168}, 2021{\natexlab{a}}.

\bibitem[Cobbe et~al.(2021{\natexlab{b}})Cobbe, Kosaraju, Bavarian, Chen, Jun,
  Kaiser, Plappert, Tworek, Hilton, Nakano, et~al.]{cobbe2021training}
K.~Cobbe, V.~Kosaraju, M.~Bavarian, M.~Chen, H.~Jun, L.~Kaiser, M.~Plappert,
  J.~Tworek, J.~Hilton, R.~Nakano, et~al.
\newblock Training verifiers to solve math word problems.
\newblock \emph{arXiv preprint arXiv:2110.14168}, 2021{\natexlab{b}}.

\bibitem[Collins(2000)]{collins2000discriminative}
M.~Collins.
\newblock Discriminative reranking for natural language parsing.
\newblock In \emph{Proceedings of the International Conference on Machine
  Learning}, 2000.

\bibitem[{Gemma Team} et~al.(2024){Gemma Team}, Mesnard, Hardin, Dadashi,
  Bhupatiraju, Pathak, Sifre, Rivi{\`e}re, Kale, Love, et~al.]{team2024gemma}
{Gemma Team}, T.~Mesnard, C.~Hardin, R.~Dadashi, S.~Bhupatiraju, S.~Pathak,
  L.~Sifre, M.~Rivi{\`e}re, M.~S. Kale, J.~Love, et~al.
\newblock Gemma: Open models based on gemini research and technology.
\newblock \emph{arXiv preprint arXiv:2403.08295}, 2024.

\bibitem[Germain et~al.(2015)Germain, Gregor, Murray, and
  Larochelle]{germain2015made}
M.~Germain, K.~Gregor, I.~Murray, and H.~Larochelle.
\newblock Made: Masked autoencoder for distribution estimation.
\newblock In \emph{International conference on machine learning}, pages
  881--889. PMLR, 2015.

\bibitem[Havrilla et~al.(2024)Havrilla, Du, Raparthy, Nalmpantis, Dwivedi-Yu,
  Zhuravinskyi, Hambro, Sukhbaatar, and Raileanu]{havrilla2024teaching}
A.~Havrilla, Y.~Du, S.~C. Raparthy, C.~Nalmpantis, J.~Dwivedi-Yu,
  M.~Zhuravinskyi, E.~Hambro, S.~Sukhbaatar, and R.~Raileanu.
\newblock Teaching large language models to reason with reinforcement learning.
\newblock \emph{arXiv preprint arXiv:2403.04642}, 2024.

\bibitem[Hendrycks et~al.(2021)Hendrycks, Burns, Kadavath, Arora, Basart, Tang,
  Song, and Steinhardt]{hendrycksmath2021}
D.~Hendrycks, C.~Burns, S.~Kadavath, A.~Arora, S.~Basart, E.~Tang, D.~Song, and
  J.~Steinhardt.
\newblock Measuring mathematical problem solving with the math dataset.
\newblock \emph{NeurIPS}, 2021.

\bibitem[Hinton(2015)]{hinton2015distilling}
G.~Hinton.
\newblock Distilling the knowledge in a neural network.
\newblock \emph{arXiv preprint arXiv:1503.02531}, 2015.

\bibitem[Hosseini et~al.(2024)Hosseini, Yuan, Malkin, Courville, Sordoni, and
  Agarwal]{hosseini2024v}
A.~Hosseini, X.~Yuan, N.~Malkin, A.~Courville, A.~Sordoni, and R.~Agarwal.
\newblock V-star: Training verifiers for self-taught reasoners.
\newblock \emph{arXiv preprint arXiv:2402.06457}, 2024.

\bibitem[Hwang et~al.(2024)Hwang, Kim, Kim, Ye, and Seo]{hwang2024self}
H.~Hwang, D.~Kim, S.~Kim, S.~Ye, and M.~Seo.
\newblock Self-explore to avoid the pit: Improving the reasoning capabilities
  of language models with fine-grained rewards.
\newblock \emph{arXiv preprint arXiv:2404.10346}, 2024.

\bibitem[Kakade and Langford(2002)]{kakade2002approximately}
S.~Kakade and J.~Langford.
\newblock Approximately optimal approximate reinforcement learning.
\newblock In \emph{Proceedings of the Nineteenth International Conference on
  Machine Learning}, pages 267--274, 2002.

\bibitem[Kakade(2001{\natexlab{a}})]{Kakade2001}
S.~M. Kakade.
\newblock A natural policy gradient.
\newblock In \emph{Advances in neural information processing systems},
  volume~14. Advances in neural information processing systems,
  2001{\natexlab{a}}.

\bibitem[Kakade(2001{\natexlab{b}})]{kakade2001natural}
S.~M. Kakade.
\newblock A natural policy gradient.
\newblock \emph{Advances in neural information processing systems}, 14,
  2001{\natexlab{b}}.

\bibitem[Kazemnejad et~al.(2024)Kazemnejad, Aghajohari, Portelance, Sordoni,
  Reddy, Courville, and Roux]{kazemnejad2024vineppo}
A.~Kazemnejad, M.~Aghajohari, E.~Portelance, A.~Sordoni, S.~Reddy,
  A.~Courville, and N.~L. Roux.
\newblock Vineppo: Unlocking rl potential for llm reasoning through refined
  credit assignment.
\newblock \emph{arXiv preprint arXiv:2410.01679}, 2024.

\bibitem[Lightman et~al.(2023)Lightman, Kosaraju, Burda, Edwards, Baker, Lee,
  Leike, Schulman, Sutskever, and Cobbe]{lightman2023let}
H.~Lightman, V.~Kosaraju, Y.~Burda, H.~Edwards, B.~Baker, T.~Lee, J.~Leike,
  J.~Schulman, I.~Sutskever, and K.~Cobbe.
\newblock Let's verify step by step.
\newblock \emph{arXiv preprint arXiv:2305.20050}, 2023.

\bibitem[Luo et~al.(2024)Luo, Liu, Liu, Phatale, Lara, Li, Shu, Zhu, Meng, Sun,
  et~al.]{luo2024improve}
L.~Luo, Y.~Liu, R.~Liu, S.~Phatale, H.~Lara, Y.~Li, L.~Shu, Y.~Zhu, L.~Meng,
  J.~Sun, et~al.
\newblock Improve mathematical reasoning in language models by automated
  process supervision.
\newblock \emph{arXiv preprint arXiv:2406.06592}, 2024.

\bibitem[Ma et~al.(2023)Ma, Zhou, Liu, Yuan, Liu, You, and Yang]{ma2023let}
Q.~Ma, H.~Zhou, T.~Liu, J.~Yuan, P.~Liu, Y.~You, and H.~Yang.
\newblock Let's reward step by step: Step-level reward model as the navigators
  for reasoning.
\newblock \emph{arXiv preprint arXiv:2310.10080}, 2023.

\bibitem[Nakano et~al.(2021)Nakano, Hilton, Balaji, Wu, Ouyang, Kim, Hesse,
  Jain, Kosaraju, Saunders, et~al.]{nakano2021webgpt}
R.~Nakano, J.~Hilton, S.~Balaji, J.~Wu, L.~Ouyang, C.~Kim, C.~Hesse, S.~Jain,
  V.~Kosaraju, W.~Saunders, et~al.
\newblock Webgpt: Browser-assisted question-answering with human feedback.
\newblock \emph{arXiv preprint arXiv:2112.09332}, 2021.

\bibitem[Ng et~al.(1999)Ng, Harada, and Russell]{ng1999policy}
A.~Y. Ng, D.~Harada, and S.~Russell.
\newblock Policy invariance under reward transformations: Theory and
  application to reward shaping.
\newblock In \emph{ICML}, volume~99, pages 278--287, 1999.

\bibitem[Ouyang et~al.(2022)Ouyang, Wu, Jiang, Almeida, Wainwright, Mishkin,
  Zhang, Agarwal, Slama, Ray, et~al.]{ouyang2022training}
L.~Ouyang, J.~Wu, X.~Jiang, D.~Almeida, C.~Wainwright, P.~Mishkin, C.~Zhang,
  S.~Agarwal, K.~Slama, A.~Ray, et~al.
\newblock Training language models to follow instructions with human feedback.
\newblock \emph{Advances in Neural Information Processing Systems},
  35:\penalty0 27730--27744, 2022.

\bibitem[Rafailov et~al.(2023)Rafailov, Sharma, Mitchell, Ermon, Manning, and
  Finn]{rafailov2023direct}
R.~Rafailov, A.~Sharma, E.~Mitchell, S.~Ermon, C.~D. Manning, and C.~Finn.
\newblock Direct preference optimization: Your language model is secretly a
  reward model.
\newblock \emph{arXiv preprint arXiv:2305.18290}, 2023.

\bibitem[Ross and Bagnell(2014)]{ross2014reinforcement}
S.~Ross and J.~A. Bagnell.
\newblock Reinforcement and imitation learning via interactive no-regret
  learning.
\newblock \emph{arXiv preprint arXiv:1406.5979}, 2014.

\bibitem[Rusu et~al.(2015)Rusu, Colmenarejo, Gulcehre, Desjardins, Kirkpatrick,
  Pascanu, Mnih, Kavukcuoglu, and Hadsell]{rusu2015policy}
A.~A. Rusu, S.~G. Colmenarejo, C.~Gulcehre, G.~Desjardins, J.~Kirkpatrick,
  R.~Pascanu, V.~Mnih, K.~Kavukcuoglu, and R.~Hadsell.
\newblock Policy distillation.
\newblock \emph{arXiv preprint arXiv:1511.06295}, 2015.

\bibitem[Setlur et~al.(2024)Setlur, Garg, Geng, Garg, Smith, and
  Kumar]{setlur2024rl}
A.~Setlur, S.~Garg, X.~Geng, N.~Garg, V.~Smith, and A.~Kumar.
\newblock Rl on incorrect synthetic data scales the efficiency of llm math
  reasoning by eight-fold.
\newblock \emph{arXiv preprint arXiv:2406.14532}, 2024.

\bibitem[Shao et~al.(2024)Shao, Wang, Zhu, Xu, Song, Zhang, Li, Wu, and
  Guo]{shao2024deepseekmath}
Z.~Shao, P.~Wang, Q.~Zhu, R.~Xu, J.~Song, M.~Zhang, Y.~Li, Y.~Wu, and D.~Guo.
\newblock Deepseekmath: Pushing the limits of mathematical reasoning in open
  language models.
\newblock \emph{arXiv preprint arXiv:2402.03300}, 2024.

\bibitem[Singh et~al.(2023{\natexlab{a}})Singh, Co-Reyes, Agarwal, Anand,
  Patil, Liu, Harrison, Lee, Xu, Parisi, et~al.]{singh2023beyond}
A.~Singh, J.~D. Co-Reyes, R.~Agarwal, A.~Anand, P.~Patil, P.~J. Liu,
  J.~Harrison, J.~Lee, K.~Xu, A.~Parisi, et~al.
\newblock Beyond human data: Scaling self-training for problem-solving with
  language models.
\newblock \emph{arXiv preprint arXiv:2312.06585}, 2023{\natexlab{a}}.

\bibitem[Singh et~al.(2023{\natexlab{b}})Singh, Blukis, Mousavian, Goyal, Xu,
  Tremblay, Fox, Thomason, and Garg]{singh2023progprompt}
I.~Singh, V.~Blukis, A.~Mousavian, A.~Goyal, D.~Xu, J.~Tremblay, D.~Fox,
  J.~Thomason, and A.~Garg.
\newblock Progprompt: Generating situated robot task plans using large language
  models.
\newblock In \emph{2023 IEEE International Conference on Robotics and
  Automation (ICRA)}, pages 11523--11530. IEEE, 2023{\natexlab{b}}.

\bibitem[Snell et~al.(2024)Snell, Lee, Xu, and Kumar]{snell2024scaling}
C.~Snell, J.~Lee, K.~Xu, and A.~Kumar.
\newblock Scaling llm test-time compute optimally can be more effective than
  scaling model parameters.
\newblock \emph{arXiv preprint arXiv:2408.03314}, 2024.

\bibitem[Sun et~al.(2017)Sun, Venkatraman, Gordon, Boots, and
  Bagnell]{sun2017deeply}
W.~Sun, A.~Venkatraman, G.~J. Gordon, B.~Boots, and J.~A. Bagnell.
\newblock Deeply aggrevated: Differentiable imitation learning for sequential
  prediction.
\newblock In \emph{International conference on machine learning}, pages
  3309--3318. PMLR, 2017.

\bibitem[Sutton and Barto(2018)]{suttonrlbook}
R.~S. Sutton and A.~G. Barto.
\newblock \emph{Reinforcement learning: An introduction}.
\newblock The MIT Press, second edition, 2018.

\bibitem[Sutton et~al.(1999)Sutton, McAllester, Singh, and
  Mansour]{sutton1999policy}
R.~S. Sutton, D.~McAllester, S.~Singh, and Y.~Mansour.
\newblock Policy gradient methods for reinforcement learning with function
  approximation.
\newblock \emph{Advances in neural information processing systems}, 12, 1999.

\bibitem[Tajwar et~al.(2024)Tajwar, Singh, Sharma, Rafailov, Schneider, Xie,
  Ermon, Finn, and Kumar]{tajwar2024preference}
F.~Tajwar, A.~Singh, A.~Sharma, R.~Rafailov, J.~Schneider, T.~Xie, S.~Ermon,
  C.~Finn, and A.~Kumar.
\newblock {Preference Fine-Tuning of LLMs Should Leverage Suboptimal, On-Policy
  Data}, 2024.

\bibitem[Uesato et~al.(2022)Uesato, Kushman, Kumar, Song, Siegel, Wang,
  Creswell, Irving, and Higgins]{uesato2022solving}
J.~Uesato, N.~Kushman, R.~Kumar, F.~Song, N.~Siegel, L.~Wang, A.~Creswell,
  G.~Irving, and I.~Higgins.
\newblock Solving math word problems with process-and outcome-based feedback.
\newblock \emph{arXiv preprint arXiv:2211.14275}, 2022.

\bibitem[Wang et~al.(2024)Wang, Li, Shao, Xu, Dai, Li, Chen, Wu, and
  Sui]{wang2024mathshepherd}
P.~Wang, L.~Li, Z.~Shao, R.~X. Xu, D.~Dai, Y.~Li, D.~Chen, Y.~Wu, and Z.~Sui.
\newblock Math-shepherd: Verify and reinforce llms step-by-step without human
  annotations, 2024.

\bibitem[Williams(1992)]{williams1992reinforce}
R.~J. Williams.
\newblock Simple statistical gradient-following algorithms for connectionist
  reinforcement learning.
\newblock \emph{Machine learning}, 8\penalty0 (3-4):\penalty0 229--256, 1992.

\bibitem[Wu et~al.(2024)Wu, Sun, Li, Welleck, and Yang]{wu2024empirical}
Y.~Wu, Z.~Sun, S.~Li, S.~Welleck, and Y.~Yang.
\newblock An empirical analysis of compute-optimal inference for
  problem-solving with language models.
\newblock \emph{arXiv preprint arXiv:2408.00724}, 2024.

\bibitem[Yu et~al.(2023)Yu, Gao, and Wang]{yu2023outcome}
F.~Yu, A.~Gao, and B.~Wang.
\newblock Outcome-supervised verifiers for planning in mathematical reasoning.
\newblock \emph{arXiv preprint arXiv:2311.09724}, 2023.

\bibitem[Yuan et~al.(2023)Yuan, Yuan, Li, Dong, Tan, and Zhou]{yuan2023scaling}
Z.~Yuan, H.~Yuan, C.~Li, G.~Dong, C.~Tan, and C.~Zhou.
\newblock Scaling relationship on learning mathematical reasoning with large
  language models.
\newblock \emph{arXiv preprint arXiv:2308.01825}, 2023.

\bibitem[Zelikman et~al.(2022)Zelikman, Wu, Mu, and Goodman]{zelikman2022star}
E.~Zelikman, Y.~Wu, J.~Mu, and N.~Goodman.
\newblock Star: Bootstrapping reasoning with reasoning.
\newblock \emph{Advances in Neural Information Processing Systems},
  35:\penalty0 15476--15488, 2022.

\bibitem[Zhang et~al.(2024)Zhang, Hosseini, Bansal, Kazemi, Kumar, and
  Agarwal]{zhang2024generative}
L.~Zhang, A.~Hosseini, H.~Bansal, M.~Kazemi, A.~Kumar, and R.~Agarwal.
\newblock Generative verifiers: Reward modeling as next-token prediction.
\newblock \emph{arXiv preprint arXiv:2408.15240}, 2024.

\end{thebibliography}

\newpage
\part*{Appendices}

\begin{appendix}

\section{Additional Related Work}
\label{app:relwork_additional}

In this section, we highlight works from four relevant streams, expanding on discussion in  Section~\ref{sec:relwork}.
First, we look at works that train verifiers to provide outcome level feedback~\citep{cobbe2021training,hosseini2024v,zelikman2022star,singh2023progprompt} on the correctness of the full response (ORM).
Here, the trained ORMs are mainly used for test-time search (best-of-$N$). 
Next, we look at works that alleviate issues with sparse feedback in ORMs, and instead train process reward models (PRMs), that can perform credit assignment. PRMs are trained either  through human annotations~\citep{lightman2023let,uesato2022solving}, or automated forms of supervision~\citep{snell2024scaling,wang2024mathshepherd,luo2024improve}.
While some works use PRMs and ORMs to collect data for supervised fine-tuning~\cite{hosseini2024v} or offline RL~\cite{setlur2024rl}, other works directly use them  as rewards in online RL~\citep{wang2024mathshepherd,uesato2022solving,shao2024deepseekmath}. 
Finally, we contrast our work against papers on imitating stronger teacher policies via RL objectives that optimize potential functions of teacher policies.

\textbf{Outcome reward models.} ORMs are verifiers~\citep{cobbe2021training, uesato2022solving} commonly used to improve the test-time performance using best-of-$N$, where we generate multiple candidate solutions from the base policy (LLM), rank them using the ORM, and pick the best one. ORMs are trained to assess correctness of a solution either using binary classification~\citep{cobbe2021gsm8k, yu2023outcome}, preference optimization using DPO~\citep{hosseini2024v}, or next-token prediction~\citep{zhang2024generative}. Furthermore, prior works train LLMs on self-generated data using ground-truth outcome rewards~($\Rex$), either via supervised fine-tuning~\citep{zelikman2022star,singh2023beyond, yuan2023scaling}, or online RL~\citep{bi2024deepseek}.
In contrast to these approaches, our work focuses on process reward models~(PRMs) for improving performance with beam-search at test time as well as online RL where we maximize the effective reward in \Eqref{eq:prm_rl_grad} which linearly combines both $\Rex$ (outcome supervision) and process supervision in the form of advantages $A^\mu$ under a prover policy $\mu$.

\textbf{PRMs and credit assignment.} Several works focus on training step-level PRMs on math reasoning tasks, either using human labels~\citep{lightman2023let} or automated LLM-generated data to estimate value functions $Q^\pi$~\citep{wang2024mathshepherd, luo2024improve}. Our work also focus on automated data collection for PRMs but empirically argues for  using the advantage function $A^\mu$ as step-level rewards along with $Q^\pi$, with a conceptual explanation in Section~\ref{subsec:adv_not_value}.
Several prior works have explored step-level search algorithms with PRMs, such as beam search~\citep{snell2024scaling}, heuristic greedy search~\citep{ma2023let}, and reward-balanced tree search~\citep{wu2024empirical}. ~\citet{hwang2024self, setlur2024rl} use advantages to identify the ``first pit'' in an incorrect reasoning trace. Specifically, they collect data by computing advantages at each step using Monte Carlo rollouts. Then in an incorrect trace, they identify the step with  the least advantage, and use the prefix of that step to construct preference pairs for offline direct preference optimization~\citep{rafailov2023direct}. 
In contrast, our work computes advantages under a prover policy, that we formally characterize, and use the computed advantages for improving test-time search and efficiency of online reinforcement learning.

\textbf{Online RL for math reasoning.} 
Once we have a trained outcome or process verifiers, it is natural update a policy by optimizing it against the learned signal, similar to how learned reward models are optimized in RLHF~\citep{ouyang2022training}.
In the context of math reasoning, \citet{uesato2022solving, havrilla2024teaching, shao2024deepseekmath} trained policies with RL, experimenting with both dense and sparse rewards. In all three works, the gains observed by using PRMs that predict step-level correctness (similar to  \cite{lightman2023let}) is quite small, compared to simply using trained ORMs, or the ground-truth outcome supervision $\Rex$. In fact, \citet{havrilla2024teaching} states that the only algorithm that does well is a form of expert iteration~\citep{anthony2017thinking}, which does not inhibit exploration as severely as some other approaches they compare with.
Our work presents one of the first results, where trained PRMs, used in conjunction with the outcome rewards during online RL, result in policies with substantially higher ($+6\%$) performance, than the one trained only with outcome supervision. Our results also indicate a $5-6 \times$ sample efficiency boost for online RL, with our trained PAVs.

\textbf{Connections to imitation learning through RL.} The idea of mixing potential functions from different policies $\mu$ and $\pi$, in order to improve upon a sub-optimal expert $\mu$ appears in~\citet{chang2015learning}, but this work  considers the structured prediction problem which is vastly different from our setting. Related to this, is the work by \citet{chang2023learning}, which uses a ``guide'' policy to rollout from prefixes generated by a base policy. The base policy can now imitate the guide by cloning those rollouts, and eventually surpass. 
Our work also uses a prover policy which can complete rollouts from states where the base policy fails. But, we also show that weak provers in many cases are able to improve the base policy, or search over its responses, better than a stronger prover policy. We tie this observation to the insight that the main goal of the prover policy is to distinguish steps taken by the base policy, as measured by advantages under the prover. Thus, we do not require the prover policy to be something better than the base policy, which is a key distinction with~\citet{chang2023learning}.

\begin{figure}
    \centering
    \includegraphics[width=0.99\linewidth]{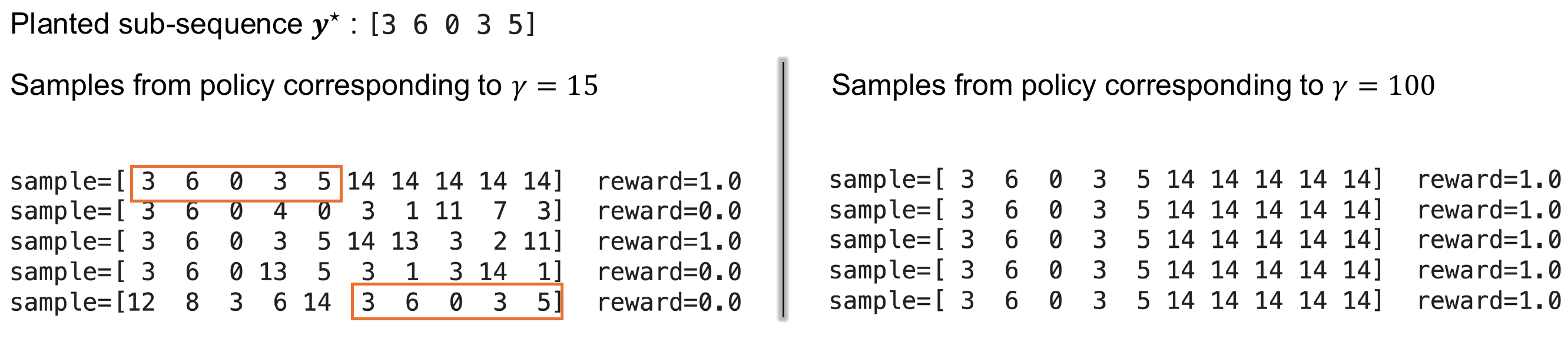}
    \caption{\textbf{\emph{Pictorial description of our planted sub-sequence didactic setup}:} An example showing five samples drawn  \textit{i.i.d.}  from a very strong policy ($\gamma=100$), and a relatively weaker ($\gamma=15$) policy in our didactic setup.}
    \label{fig:didactic_example}
\end{figure}

\section{Didactic Analysis}
\label{app:toy_additional}

 We consider sequences of length 10 from a 15-token vocabulary $\mathcal{V} \coloneqq \{1,2, \ldots, 14\}$, where the end-of-sequence token is given by $14$, and all tokens following the end-of-sequence token (including it) are masked. Given an unknown planted sequence $\by^\star$ (in~\Figref{fig:didactic_example}), we train a policy $\pi$ with policy gradient, where the outcome reward we wish to optimize is  terminal and sparse, \textit{i.e.}, for $\by \sim \pi$ we have $r(\by, \by^\star) = 1$ if and only if $\by^\star$ appears in $\by$, and $0$ otherwise (\Figref{fig:didactic_example}). The policy $\pi$ in our experiments is represented by a multi-layer neural network, similar to the MADE architecture~\citep{germain2015made}. The prover policy $\mu$  is parameterized by a scalar $\gamma > 0$. In particular, at any state $\bs$, where the last $k$ tokens leading up to $\bs$ match first $k$ tokens of $\by^\star$, then:  $$\mu(\by^\star_{k+1} \mid \bs) \propto \gamma,$$ and uniform on all other tokens. Thus, as $\gamma$ increases, the performance of  $\mu$ improves and  $\rightarrow 1$ as $\gamma \rightarrow \infty$. For our experiments, we assume (almost) oracle access to ground-truth advantage and $Q$-values, thus mitigating any confounding issues due to usage of a learned verifier. We are able to approximate exact $Q$-values very accurately by using Monte Carlo estimates with large $>100$ rollouts. With the goal of optimizing the terminal reward $r(\by, \by^\star)$, we optimize $\pi$ with two types of rewards: (i) only the outcome reward $r(\by, \by^\star)$, which is equivalent to using only $Q^\pi$ as step-level rewards; and (ii) using the effective reward: $Q^\pi + \alpha A^\mu$ as the step-level reward. 
 
\textbf{Training details.} We use effective rewards with $\alpha=1$, and use the gradient in \Eqref{eq:prm_rl_grad} to update the policy via policy gradient iterations. For the ORM runs, where we only use the outcome reward $r(\by, \by^\star)$, the policy gradient is equivalent to the case where $\alpha=0$ in \Eqref{eq:prm_rl_grad}. We train for 10,000 iterations in both cases, with a batch size of 64, and a constant learning rate of $1e-3$ for the Adam optimizer. The RL runs are initialized with a supervised finetuned policy. For this we take a randomly initialized network, based on the MADE architecture~\citep{germain2015made}, with 3 layers, and 128 hidden units in each. Then we train it with supervised next-token prediction loss for 50 iterations on a dataset of $3200$ samples from a weak policy $(\gamma=5.0)$. The batch size for the SFT training is also set to 64. For evaluating Pass @$N$ performance, we either sample $N$ independent trajectories (temperature 1.0) from the base policy trained using effective rewards, or only $Q^\pi$. We also evaluate Pass @$N$ for the SFT policy for comparison.

\section{Additional: Experiments on Test-time Search with PAVs}
\label{app:test_time_scaling_additional}

\textbf{Implementation details.} For our experiments in \Secref{sec:prms_for_bon}, we use three pretrained models: Gemma 2B, 9B and 27B.
We finetune each of these on the MATH~\citep{hendrycksmath2021} dataset. The finetuning is done for 5000 iterations, with a batchsize of 32, and a maximum learning rate of $5e-6$ for 2B, 9B and $5e-7$ for the 27B models. 
We trained the policies using the Adam optimizer, with a linear warm up and cosine decay learning rate schedule. The linear warm up is done for the first 500 iterations.
For the base policies, we choose the SFT checkpoints with the best accuracy on a holdout validation set of the MATH dataset.
Given the SFT checkpoints, we next train PAVs using the procedure in \Secref{subsec:dataset}. We do this for a class of provers, which include the base policies themselves. As we discuss in \Secref{sec:prms_for_bon}, the prover class also includes the best-of-$K$ policy for $K$ in $\{2, 4, 8, 16, 32\}$. 

We use the hold out validation set to ascertain the value of $\alpha$ in the effective reward. For each base policy we run beam search with a beam size of 16 on this hold out validation set, and using the base policy itself as a prover, we evaluate the value of $\alpha$ that works best in the effective reward. We find that $\alpha=0.5$ worked best for Gemma 2B and 9B base policies, while a lower value of $\alpha=0.2$ was optimal for Gemma 27B. To tune $\alpha$ we ran a grid search over the range $[0.0, 1.0]$, evaluating at an interval of $0.1$. We observe that the choice of $\alpha$ is a relatively robust one, since for all three base policies, we saw improvements (over only $Q^\pi$  as the reward) for values in the range of $[0.2, 0.6]$. Having a separate value of $\alpha$ for each base policy, we use the same value in the effective reward given by any choice of the prover policy that is used for that base policy. Next, we present an experiment that compares the predictive power of effective reward vs. just $Q^\pi$ at initial states of a rollout under the base policy $\pi$, when either is used to  predict the final outcome given by $\Rex$.   

\textbf{Experiment: Is the effective reward able to predict the final outcome better than $Q^\pi$?} 
In~\Figref{fig:test_statistic}, we describe an experiment where for both the effective reward $Q^\pi + \alpha A^\mu$ (PAV) and just $Q^\pi$ (PQV), we compute the error of the classifier that makes a prediction on the final outcome by thresholding on either reward value at each step of the rollout. This threshold is computed using a validation set, and is separate for each step and reward combination. The figure tells us that the outcome prediction error drops for both rewards as the base policy is rolled out more, but clearly the effective reward dominates $Q^\pi$ (PQV) across all steps. Thus,  the effective reward is a more informative signal (lower classification error) for the problem of predicting the success of a partial rollout, especially in the earlier steps of the rollout. This helps to explain  the better performance of  beam search with a finite capacity beam that re-ranks partial rollouts with the effective reward. For this experiment, we use the Gemma 9B SFT policy as the base policy and the best-of-$4$ policy corresponding to the same SFT policy as the prover. 

\begin{figure}[!ht]
    \centering
    \includegraphics[width=0.45\linewidth]{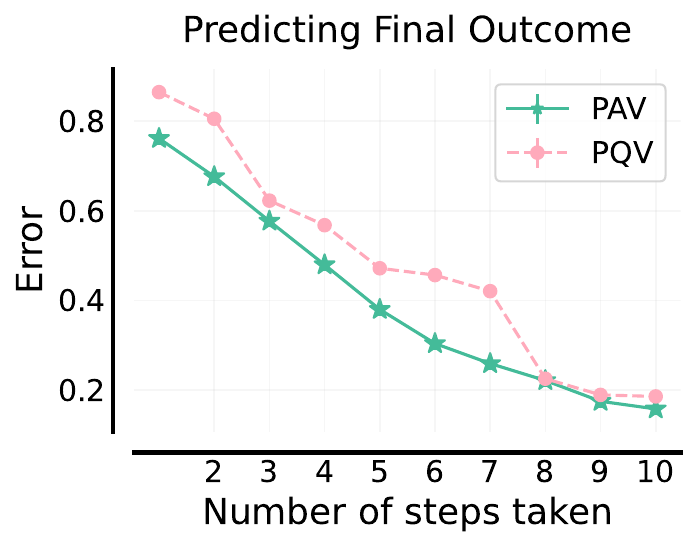}
    \caption{\textbf{\emph{Effective rewards at any step are able to predict the outcome rewards at the final step, better than $\mathbf{Q^\pi}$}:}For both the effective reward $Q^\pi + \alpha A^\mu$ and just $Q^\pi$, we compute the error of the classifier that makes a prediction on the final outcome by thresholding on either reward value at each step of the rollout. This threshold is computed using a validation set, and is separate for each step and reward combination. The figure tells us that the outcome prediction error drops for both rewards, as the base policy is rolled out more, but compared to $Q^\pi$ at an intermediate step, the effective reward $Q^\pi + \alpha A^\mu$ at an intermediate step is able to reliably predict the final outcome level correctness of the future rollout under the base policy $\pi$.}
    \label{fig:test_statistic}
\end{figure}

\section{Details on Collecting Data and Training PAVs}
\label{app:dataset_pavs}

In \Figref{fig:scaling_panel_2}(b) in \Secref{subsec:dataset}, our seed rollouts are \textit{i.i.d.} sampled from $\pi$, but prior works~\citep{luo2024improve,setlur2024rl,hwang2024self} employed a 
``first pit'' strategy for coverage. Here, some initial sampling budget is spent to identify ``high value'' states where $Q^\pi$ is larger than a threshold. Then, for any incorrect sample from these high value states greater budget is spent to estimate the first step (first pit) with $Q^\pi=0$. All prefixes (and their estimated $Q$-values) until the first pit are then added to the training data. 
In \Figref{fig:first_pit}, we compare beam search using PAVs trained using data from the first pit strategy, and the random sampling strategy. Both of them use the best value of $\nicefrac{n_\mathrm{mc}}{n_\mathrm{cov}}$ from \Figref{fig:scaling_panel_2}(b) for every dataset size. We find the first pit strategy to be better than random, especially when the number of seed rollouts are limited. Once we get coverage over such pits, we sample a large number of partial rollouts conditioned on each prefix until the first pit. This is used to compute the Monte Carlo estimate of $Q$ values more accurately on the path to the first pit. Each prefix and estimated $Q$ value pair is then added to the dataset used to train PAVs. 

\textbf{Training details.} All PAVs used in this work are trained by taking the Gemma 9B pretrained checkpoint and finetuning it on the data collected from the above strategy. The data collection uses first pit strategy for better coverage over pits in the seed rollouts. Based on findings in \Figref{fig:scaling_panel_2}(b), we use a high value of $n_{\mathrm{mc}}=20$ to estimate the $Q$-values accurately for each step in the seed rollout. For each base policy, in total, we collect a dataset of over $300,000$ (prefix, $\hat{Q}^\pi$-value) pairs. Here, $\hat{Q}^\pi$ is the Monte Carlo estimate for the $Q$-value at the prefix, under the policy $\pi$. on which we finetune the Gemma 9B model with cross-entropy loss. Since the distribution of values for $\hat{Q}^\pi$ can be skewed, we split the range of $\hat{Q}^\pi$-values into two buckets, based on which we also partition the training data. The first bucket is the set of all prefixes with $\hat{Q}^\pi < 0.5$ and the second is the set of all prefixes with $\hat{Q}^\pi \geq 0.5$. Then, we use class-balanced sampling over these buckets to finetune the pretrained model for 20000 training iterations, using a batch size of $32$. We use an Adam optimizer with a maximum learning rate of $5e-7$. We use a linear warm up (till 2000 steps), followed by a cosine decay learning rate schedule to train the models. Since  a pretrained LLM would output a matrix of logits (vocabulary size $\times$ sequence length) we fix a token as the ``scoring token'' to be the end of the sequence / prefix that needs to be scored. The logits of this scoring token are then used to determine the prediction for the LLM being trained. 

Once we have models that predict the $Q^\pi$ for a base policy $\pi$, we compute the $Q$-value under the $\mathrm{BoK}$ policy corresponding to $\pi$, by setting: $Q^{\mathrm{BoK}(\pi)}(\bs, a) = 1-(1-Q^\pi(\bs, a))^{K}$. Next, given the $Q$-values we for $\pi$, and their corresponding best-of-$K$ policies, we can compute any effective reward in \Eqref{eq:prm_rl_grad}, using the definition of advantage value of a step in \Eqref{eq:advantage}.

\begin{figure}[!ht]
    \centering
    \includegraphics[width=0.45\linewidth]{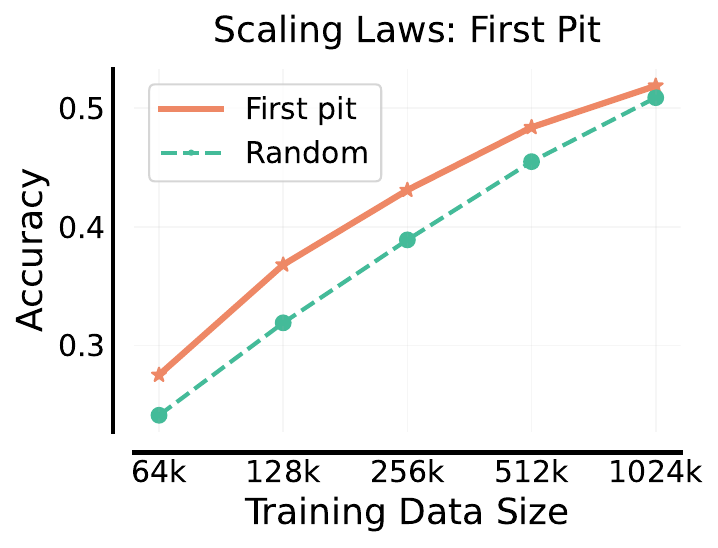}
    \caption{\textbf{\emph{First pit strategy from \citet{luo2024improve,setlur2024rl}}}: 
    We compare the beam search performance (with beam size 128) for a PAV trained on data collected using two types of seed rollouts. For the seed rollouts, we either randomly sample from the distribution of state action pairs induced by the base policy. Or we improve coverage particularly by using the ''first pit'' strategy of identifying the first state with low $Q$ values on a partial rollout that starts from a high $Q$-value state and ends with an outcome reward of $0$, i.e., $Q$ value is 0. 
    } 
    \label{fig:first_pit}
\end{figure}

\section{Additional: Experiments on RL Training with PAVs}
\label{app:rl_with_pavs_additional}

\textbf{Training details.} As discussed in Section~\ref{sec:prms_for_RL}, the initialization for RL training is the RFT (rejection finetuned) checkpoint for the corresponding base policies. More specifically, we consider two base policies Gemma 2B SFT, and Gemma 9B SFT, where the RL training is initialized with the policy obtained by further optimizing the base policy via rejection finetuning (RFT)~\cite{yuan2023scaling}. This is done to improve coverage over states and actions during the initial iterations of RL training. 
For rejection finetuning we train the SFT model (base policy) on all correct trajectories sampled when collecting seed rollouts for training PAVs (that predict the advantage of the same base policy). The training details for RFT remain same as SFT and are detailed in Appendix~\ref{app:test_time_scaling_additional}. 
We use the REINFORCE~\citep{sutton1999policy} algorithm to improve the base policy.
The RL training is run for 10000 iterations for the 2B model, and for 5000 iterations for the 9B model. For both we use the Adam optimizer with a learning rate of 1e-7, and a batchsize of 32. The maximum response length is set to 512. For both we used a learning rate schedule with a linear warm up for 10\% of the total training iterations, followed by a cosine decay. The  implementation of our policy gradient algorithm also uses a token-level value function that is initialized to be the base policy itself, and is trained with a square loss. The value function is only used as a baseline during RL training, i.e., at any state $\bs$ it is only predicting $\Exp_{a \sim \pi(\cdot \mid \bs)} Q^\pi(\bs, a) + \alpha A^\mu(\bs, a)$. 

Most importantly, we use a validation set to identify a good choice of $\alpha$ in the effective reward (in \Eqref{eq:prm_rl_grad}). For Gemma 2B this value is $5.0$ and for the 9B policy $\alpha=3.0$ works best. 
Similar to the choice of $\alpha$ for test-time search, we find that most values of $\alpha$ in the range of $0.5$ to $6.0$ improved performance over ORM-RL, to different degrees.
Both policies are also optimized with KL regularization, against the initial iterate of the RL policy, where the strength of the KL penalty is set to 0.001 for both.

\section{Theoretical Analysis:  Complementary Provers Amplify Base Policy Improvement}
\label{app:theory-convergence}

In this section, we present the proofs for our theoretical results in the main paper (\Secref{subsec:theory}). We begin by describing some notation we use in our proofs,  the natural policy gradient algorithm we use for the  policy update, followed by the proof for Theorem~\ref{thm:policy-improvement}. We also present a simple application of this result in Proposition~\ref{prp:learning-signal}. Our results in this section are in the tabular setting, with softmax parameterization of the policies. Note that for the deterministic Markov decision process induced by the LLM, we are indeed in a tabular setting, where the set of states and actions is discrete, but large and grows exponentially in sequence length. 

\textbf{Notation and preliminaries.} We use $d^\pi_{h}, d^\mu_{h}$ to denote the distribution over states $\bs_h$ at time step $h$, starting from the initial state distribution given by the empirical distribution over the questions in the dataset $\mathcal{D}$, and following the base policy $\pi$, or prover policy $\mu$ respectively. 
The term $d^\pi_{\bs}$ denotes the distribution over future states, starting from state $\bs$, and following policy $\pi$. Here, $\bs$ can be a state at any time $h \in [0, \ldots, H]$. For convenience, we overload the notation $d^\pi_{\bs}$ (the distribution over future states induced by a policy starting from state $\bs$), 
and use $d^\pi_{\rho}$ to denote the mixture distribution over  $d^\pi_{\bs}$ starting from a random state $\bs$ drawn from $\bs \sim \rho$, and following policy $\pi$.

The term $Q^\pi(\bs_h, a_h)$ refers to the value of state-action pair $\bs_h, a_h$, \textit{i.e.}, the expected return in the future, starting the policy from state $\bs_h$ and taking action $a_h$:
\begin{align}
    \label{eq:q-function}
    Q^{\pi}({\bs_h}, {a_{h}}) \coloneqq{
    {\Exp}_{
    {a_{h}, \ldots, a_H \sim {\pi}(\cdot|\bs_{h}, a_h)}
    } \Big[ \mathrm{Rex}\left( (a_1, \ldots, a_H), \by^{\star}_{\bx}\right)\Big]}.
\end{align}
Note that $\by^\star_{\bx}$ is known on the dataset $\mathcal{D}$, and state $\bs_h$ contains the question $\bx$ as part of it.
Similarly, we can define the value function $V^{\pi}({\bs_h})$ of a state $\bs_h$ as:
\begin{align}
    \label{eq:v-function}
    V^{\pi}({\bs_h}) \coloneqq{
    {\Exp}_{
    {a_{h+1} \sim {\pi}(\cdot|\bs_{h})}
    } Q^\pi(\bs_h, a_h)}.
\end{align}
The advantage function is then given by:
\begin{align}
    \label{eq:a-function}
    A^\pi(\bs_h, a_h) \coloneqq Q^\pi(\bs_h, a_h) - V^\pi(\bs_h).
\end{align}
The policy gradient algorithm we use to update the base policy iteratively is natural policy gradient~\citep{kakade2001natural}, and we use $\pi_t$ to refer to the base policy iterate at time $t$ of this iterative algorithm. Finally, we use $\mathcal{S}$ to denote the set of all states (prefixes) and $\mathcal{A}$ for the set of all actions (steps) that the LLM can take at any state.

\textbf{Parameterization of the base policy.} We adopt the softmax parameterization for the base policy: 
\begin{align}
    \pi_\btheta(a\mid \bs_h) = \frac{\exp(\btheta_{\bs_h,a})}{\sum_{a' \in \mathcal{A}} \exp(\btheta_{\bs_h,a'})}.
\end{align}
Here $\btheta_{\bs_h, a} \in \mathbf{\Theta} \subseteq \R^d$ controls the probability of taking action $a$ at state $\bs_h$. The full set of parameters across all states and actions is denoted by $\btheta \in \R^{d \times |\mathcal{S}| \times |\mathcal{A}|}$. Whenever clear from context, we overload the notation $\pi_t$ to denote both the policy at iterate $t$, \textit{i.e.},  $\pi_{\btheta_t}$ and the parameter $\btheta_t$ itself. \textit{E.g.}, the gradient operator $\nabla_{\pi_t}[\cdot]$ is referring to $\nabla_{\btheta_t}[\cdot]$.

\textbf{Defining policy improvement.} Let $\rho$ be a distribution over all states $\{\bs_h :  h \in [0, 1, \ldots, H]\}$, then $\Exp_{\bs \sim \rho} V^\pi(\bs)$, and $\Exp_{\bs \sim \rho} V^\mu(\bs)$ give us the expected value functions over states across time steps, measured under $\rho$, for policies $\pi$ and $\mu$ respectively. We assume that $d^\pi_h $ and $d^\mu_h$ are both absolutely continuous with respect to $\rho$, and use the expected value function over $\rho$ as the quantity we track before and after a policy update. 
A positive change in $\Exp_{\bs \sim \rho} V^\pi(\bs)$ implies a net positive improvement in the base policy. Thus, progress is made at each update of the policy when:
\begin{align*}
    \Exp_{\bs \sim \rho} V^{\pi_{t+1}}(\bs) \; - \;  \Exp_{\bs \sim \rho}  V^{\pi_{t}}(\bs) > 0.
\end{align*}

\subsection{Natural Policy Gradient}
\label{subsec:bkgnd-natural-policy-gradient}

The natural policy gradient (NPG) algorithm~\citep{Kakade2001} defines a Fisher information matrix
(induced by the policy), and performs gradient updates in the geometry
induced by the following matrix:
\begin{align}\label{eqn:npg-1}
  F_\rho(\pi) &=\Exp_{\bs \sim d^{\pi}_{\rho}}
\E_{a\sim \pi(\cdot \mid \bs) }\Big[ \nabla_\pi \log
\pi(a \mid \bs) \Big(\nabla_\pi \log \pi(a\mid
\bs)\Big)^\top \Big] 
\end{align}
Typically, the NPG update does gradient updates on the objective $\ell_{\mathrm{ORM-RL}}$ in \Eqref{eq:standard_rl}, but in our case, the objective of interest is $\ell_{\mathrm{PAV-RL}}$ in \Eqref{eq:prm_rl}, and thus the natural gradient is given by:
\begin{align}
    \label{eq:-npg-2}
    \pi_{t+1} &= \pi_{t} + \gamma \cdot F_\rho(\pi^{t})^\dagger \paren{\nabla_\pi \ell_{\mathrm{PAV-RL}}(\pi) \Big\vert_{\pi = \pi_t}} ,
\end{align}
where $M ^\dagger$ denotes the Moore-Penrose pseudoinverse of the
matrix $M$.  
We restrict to using the initial state distribution $\rho$ in our update rule, \textit{i.e.}, we restrict attention to states
$\bs$ reachable from $\rho$, since $\rho$ governs the performance measure of interest when evaluating the expected value of a policy. Thus, without loss of generality,
we can exclude states that are not reachable under $\rho$. Specifically, we restrict the MDP to the set:
  $\{\bs_h~:~\exists \pi ~~\text{such that}~~ d^\pi_\rho(\bs_h) >
  0, h \in [0, \ldots, H]\}$. The scalar $\gamma > 0$ determines the learning rate.

\subsection{Useful Lemmas}
\label{subsec:useful-lemmas}

\begin{lemma}
\label{lem:pdl}
[The performance difference lemma;~\citep{kakade2002approximately}]
For all policies $\pi, \pi^\prime$ and states
$s_0$,
\begin{eqnarray*}
  V^\pi(\bs) - V^{\pi^\prime}(\bs)
&=& \Exp_{\bs_h \sim d_{\bs}^\pi }\Exp_{a_h\sim \pi(\cdot\mid \bs_h) }
\left[A^{\pi^\prime}(\bs_h,a_h)\right].
\end{eqnarray*}
\end{lemma}
\begin{proof}
See proof of Lemma 6.1 in \citet{kakade2002approximately}.
\end{proof}

\begin{lemma}[Natural policy gradient update]
For the natural policy gradient in \Eqref{eq:-npg-2}, the corresponding policy update is given by:
\label{lem:npg-update-derivation}
\begin{align}
    \pi^{t+1}(a \mid \bs) &= \pi^{t}(a \mid \bs) \cdot \frac{\exp\paren{\gamma \cdot \paren{{Q^{{t}}(\bs, a) + \alpha \cdot A^\mu(\bs, a) }}}}{Z^{t}(\bs)}, \\
    Z^{t}(\bs) &=   \gamma \cdot \sum_{a \in \mathcal{A}}  \paren{{Q^{t}(\bs, a) + \alpha \cdot A^\mu(\bs, a)} }
\end{align}
\end{lemma}
\begin{proof}
We use arguments similar to the proof of Lemma 15 in \cite{agarwal2021theory}, with the key difference of separately accounting for the term $A^\mu(\bs, a)$ in the effective reward. For the sake of completeness we reproduce some of the derivation, accounting for the $A^\mu$ term in the process. We follow compatible function approximation
in~\citet{sutton1999policy} and ~\citet{Kakade2001}. For a vector $\bw \in \R^{d \times |\mathcal{S}||\mathcal{A}|}$, we define the error function
\begin{align}
\label{eq:proof-20}
      L^\pi(\bw) = \Exp_{\bs\sim d^{\pi}_\rho}, \Exp_{a\sim \pi(\cdot \mid \bs)}\brck{\bw^\top \nabla_\pi \log \pi(\cdot \mid \bs) - \paren{A^{\pi}(\bs,a) + \alpha A^{\mu}(\bs,a)- \alpha \E_{a\sim\pi_t(\cdot \mid \bs)} A^\mu(s, a)}}^2.
\end{align}

Let $\bw^\star$ be the minimizer of $L^\pi(\bw)$ with the
smallest $\ell_2$ norm. Then by definition of Moore-Penrose
pseudoinverse:
\begin{align}\label{eq:proof-21}
    \bw^\star &= F_\rho(\pi)^\dagger \E_{\bs\sim d^{\pi}_\rho, a\sim \pi(a | \bs)} \brck{\nabla_\pi \log \pi(a|s) \paren{A^{\pi}(\bs,a) + \alpha A^{\mu}(\bs,a)- \alpha \E_{a\sim\pi_t(\cdot \mid \bs)} A^\mu(s, a}} \nonumber \\
    &=  F_\rho(\pi)^\dagger \nabla_\pi \ell_{\mathrm{PAV-RL}}(\pi).
\end{align}
In other words, $\bw^\star$ is precisely proportional to the NPG update direction. Note further that for the Softmax policy parameterization, we have:
\[
    \bw^\top \nabla \log \pi(a|s) = \bw_{s,a} - \sum_{a'\in \mathcal{A}} \bw_{s,a'} \pi(a'|s).
\]
Since $\sum_{a\in \mathcal{A}} \pi(a|s) A^{\pi}(s,a) = 0$, this immediately
yields that: $$L^\pi(A^{\pi}(\bs, a) + \alpha A^\mu(\bs, a)) = 0.$$ However, this might not be
the unique minimizer of $L^\pi$, which is problematic since
$\bw^\star(\pi)$ as defined in terms of the Moore-Penrose
pseudoinverse is formally the smallest norm solution to the
least-squares problem, which $A^{\pi} + \alpha A^\mu $ may not be. However,
given any vector $\bv\in \R^{ |\mathcal{S}| \times |\mathcal{A}|}$, let us consider solutions of the form $A^{\pi} + \alpha A^\mu  + \bv$. Due to the form of the derivatives of the policy for the softmax parameterization, we have for any state $\bs,a$ such that $\bs$ is reachable under $\rho$,
\[
\bv^\top \nabla_\pi \log \pi(a\mid\bs) = \sum_{a'\in\mathcal{A}} (\bv_{\bs,a'}\mathbf{1}[a = a'] - \bv_{\bs,a'}\pi(a'\mid\bs)) = \bv_{\bs,a} - \sum_{a'\in\mathcal{A}} \bv_{\bs,a'}\pi(a'\mid \bs).
\]
 This is because $\pi$ is a stochastic policy with $\pi(a\mid \bs) > 0$ for all actions $a$ in each state $\bs$, so that if a state is reachable under $\rho$, it will also be reachable using $\pi$, and hence the zero derivative conditions apply at each reachable state.
For $A^{\pi} + \alpha A^\mu + \bv$ to minimize $L^\pi$, we would like $v^\top \nabla_\pi \log \pi(a\mid \bs) = 0$ for all $\bs,a$ so that $\bv_{\bs,a}$ is independent of the action and can be written as a constant $c_{\bs}$ for each $\bs$ by the above equality. Hence, the minimizer of $L^\pi(\bw)$ is determined up to a state-dependent offset, and
\[
    F_\rho(\theta)^\dagger \nabla_\pi \ell_{PAV-RL}(\pi) = {Q^{\pi}} + \alpha A^\mu + \bv,
\]
where $\bv_{\bs,a} = c_{\bs}$ for some $c_{\bs}\in \R$ for each state $\bs$ and action $a$. Finally, we observe that this yields the updates
\[
\btheta_{t+1} = \btheta^{t} + \gamma ({Q^{\pi}} + \alpha A^\mu + \bv) \quad \mbox{and}\quad \pi_{t+1}(a \mid \bs) = \pi_{t}(a \mid \bs) \frac{\exp(\gamma Q^{t}(s,a) + \gamma \alpha A^{\mu}(s,a))}{Z^t(\bs)}.
\]
Owing to the normalization factor $Z^t(\bs)$, the state dependent offsets cancel in the updates for $\pi$, which yields the statement of the lemma.
\end{proof}

\subsection{Proof of Theorem~\ref{thm:policy-improvement}}
\label{subsec:thm-polic-improv-proof}

\begin{proof}

For some notational convenience, we use $V^t, V^{t+1}$,  to denote the value functions  $V^{\pi_{t}}, V^{\pi_{t+1}}$ for policies at $\pi_t$, $\pi_{t+1}$ at time $t$ and $t+1$ respectively. Similarly, we use $ A^{t}, A^{t+1}$ for $A^{\pi_{t}}, A^{\pi_{t+1}}$ respectively. For the distribution over states $d^{\pi_{t+1}}_\rho$ induced by the policy $\pi_{t+1}$, starting from an initial distribution of states given by $\rho$, we simplify the notation and use $d^{t+1}_\rho$. Similarly we use $d^{t}_\rho$ for $d^{\pi_{t}}_\rho$.

Next, for simplicity we set $\alpha=1$ in the natural policy gradient update in Lemma~\ref{lem:npg-update-derivation}. It is easy to see that the lower bound result we show holds for any value of $\alpha >0$, and the term in the lower bound scales linearly with $\alpha$. Note, that this is not a free variable, and $\alpha$ has to be $O(1)$, since as we increase the value of $\alpha$ we would have to correspondingly reduce $\gamma$ for our result to hold. For now, we fix $\alpha=1$.

From the policy difference Lemma~\ref{lem:pdl}, we can write: 
\begin{align}
    \label{eq:proof-1}
    \E_{\bs\sim \rho} V^{t+1}(\bs) -  \E_{\bs\sim \rho} V^{t}(\bs) & = \E_{\bs \sim d^{t+1}_\rho} \E_{a \sim \pi^{t+1}(a \mid \bs)} \brck{ A^{t}(\bs, a) }
\end{align}
Next, from the natural policy gradient update in Lemma~\ref{lem:npg-update-derivation}, we can write $A^{t+1}(\bs, a)$ as:
\begin{align}
    \label{eq:proof-2}
    A^{t} (\bs, a) =  \frac{1}{\gamma} \cdot \log \paren{\frac{\pi^{t+1} (a\mid \bs) \cdot Z^t(\bs, a)}{\pi^t(a \mid \bs)}} - A^\mu(\bs, a)
\end{align}
Substituting \Eqref{eq:proof-2} in \Eqref{eq:proof-1} we get:
\begin{align}
    \label{eq:proof-3}
    \E_{\bs\sim \rho} V^{t+1}(\bs) & -  \E_{\bs\sim \rho} V^{t}(\bs) =  \frac{1}{\gamma} \E_{\bs \sim d^{t+1}_\rho} \brck{\mathrm{KL}(\pi^{t+1}(\cdot \mid \bs)) \| \pi^{t}(\cdot \mid \bs))} \nonumber\\ 
    &+ \frac{1}{\gamma} \log Z^t(\bs) - \E_{\bs \sim d^{t+1}_\rho} \E_{a \sim \pi^{t+1}(a \mid \bs)} A^\mu(\bs, a). 
\end{align}
Recall that for $\alpha=1$,
\begin{align}
    \label{eq:proof-4}
    \log  Z^t(\bs) = \log  \E_{a\sim \pi^{t}(\cdot \mid \bs)} \exp\paren{\gamma \cdot ({A^{{t}}(\bs, a) +  A^\mu(\bs, a)})}
\end{align}
Applying Jensen's inequality we get:
\begin{align}
    \label{eq:proof-5}
    \log  Z^t(\bs) &\geq      {\gamma \cdot \E_{a\sim \pi^{t}(\cdot \mid \bs)} \brck{{A^{{t}}(\bs, a) +  A^\mu(\bs, a)}}} \\ 
      &=   \gamma \cdot   \E_{a \sim \pi^{t}(\cdot \mid \bs)} \brck{A^\mu(\bs, a)} ,
\end{align}
since $\E_{\pi^t} \brck{ A^{{t}}(\bs, a) } = 0$. Note that in \Eqref{eq:proof-4} the KL term is always non-negative. Thus, we can lower bound our policy improvement:
\begin{align}
    \label{eq:proof-6}
    \E_{\bs\sim \rho} V^{t+1}(\bs) & -  \E_{\bs\sim \rho} V^{t}(\bs) \geq \E_{\bs \sim d^{t+1}_\rho} \langle \pi^{t+1} (\cdot \mid \bs ) - \pi^t(\cdot \mid \bs), A^\mu(\bs, \cdot) \rangle,
\end{align}
where the inner product is the standard euclidean product as our actions space $\mathcal{A}$ is discrete. 

In the following we will treat the distribution $\pi^{t+1}(\cdot \mid \bs)$ as a vector denoted by $\pi^{}$ Next, from the NPG update we know that:
\begin{align}
    \label{eq:proof-7}
    \pi^{t+1}(a\mid \bs)-\pi^{t}(a\mid \bs) = \pi^t(a\mid \bs) \paren{\frac{\exp{\paren{\gamma A^t(\bs, a) + \gamma A^\mu(\bs, a)}}}{Z^t(\bs)}-1}
\end{align}

We note that for $\gamma \ll 1$, $\exp{\gamma(A^t(\bs, a)) + \gamma(A^\mu(\bs, a))} = \Theta(1+ \gamma (A^t(\bs, a) + A^\mu(\bs, a)))$, where the terms that grow as $\omega(\gamma)$ are being ignored. 
Based on this, for $\gamma \ll 1$, $\exists$ constants $0 < C_1 < C_2$ such that:
\begin{align*}
    \exp(\gamma A^t(s,a) + \gamma A^\mu(s,a)) - 1 \in [C_1 \gamma(A^t(s,a) + A^\mu(s,a)), C_2 \gamma(A^t(s,a) + A^\mu(s,a))]
\end{align*}
Applying the above claim in \Eqref{eq:proof-7} gives us:
\begin{align}
    &\pi^{t+1}(a\mid \bs)-\pi^{t}(a\mid \bs) \geq  \pi^t(a\mid \bs) \paren{\frac{1+C_1 \gamma (A^t(\bs, a) + A^\mu(\bs, a))}{1+C_2\gamma \E_{a\mid \pi^{t}(\cdot \mid \bs)} \brck{A^t(\bs, a) + A^\mu(\bs, a)} }-1} \nonumber \\
    &\geq C_3  \gamma  \frac{ 
        \paren{\pi^t(a \mid \bs)\paren{A^t(\bs, a) + A^\mu(\bs, a)} - \pi_t(a\mid s) \E_{a\sim \pi_t(a\mid s)} \brck{A^t(\bs, a) + A^\mu(\bs, a)} } 
    }
    {
        1+ C_2\gamma  \E_{a\sim \pi_t(a\mid s)} \brck{A^t(\bs, a) + A^\mu(\bs, a)}
    } \\
     &= C_3  \gamma  \frac{ 
        \paren{\pi^t(a \mid \bs)\paren{A^t(\bs, a) + A^\mu(\bs, a)} - \pi_t(a\mid s) \E_{a\sim \pi_t(a\mid s)} \brck{ A^\mu(\bs, a)} } 
    }
    {
        1+ \gamma C_2 \E_{a\sim \pi_t(a\mid s)} \brck{A^t(\bs, a) + A^\mu(\bs, a)},
    }  \label{eq:proof-8}
\end{align}
where we reused: $\E_{\pi^t} \brck{ A^{{t}}(\bs, a) } = 0$. Here, $C_3 > 0$ is a constant. 

We now plug in the  above lower bound  into \Eqref{eq:proof-5} to get the final lower bound on the policy improvement in Theorem~\ref{thm:policy-improvement}. For this, we will once again use the assumption that the learning rate $\gamma \ll 1$, which allows us to use $1+ \gamma  \E_{a\sim \pi_t(a\mid s)} \brck{A^t(\bs, a) + A^\mu(\bs, a)} \geq C_4$ for some constant $C_4 > 0$. This is because, in our setting the range of the advantages is $[-1, 1]$.  
\begin{align}
    \E_{\bs\sim \rho}\brck{ V^{t+1}(\bs) -  V^{t}(\bs)} &\gsim \gamma {
    \E_{\bs\sim d^{t+1}_\rho} 
    \brck{
        \E_{a\sim \pi^t(a \mid \bs)} \brck{A^\mu(\bs, a)  A^t(\bs,a )}
        } 
    }  \nonumber\\
    +& \gamma {
    \E_{\bs\sim d^{t+1}_\rho} 
    \brck{
        \E_{a\sim \pi^t(a \mid \bs)} \brck{\paren{{A^\mu}(\bs, a)}^2}
        } 
    }  \nonumber \\
    -& \gamma {
    \E_{\bs\sim d^{t+1}_\rho} 
    \brck{
        \paren{\E_{a\sim \pi^t(a \mid \bs)} \brck{A^\mu(\bs, a)}}^2
        } 
    }  \label{eq:proof-9}
\end{align}
Now \Eqref{eq:proof-9} gives us,
\begin{align}
    \label{eq:proof-10}
    \E_{\bs\sim \rho}\brck{ V^{t+1}(\bs) -  V^{t}(\bs)} &\gsim
    \gamma   \E_{\bs\sim d^{t+1}_\rho}  \brck{\mathbb{V}_{a\sim \pi_t(a\mid \bs)} \brck{A^\mu(\bs, a)} - \E_{a\sim \pi_t(a\mid \bs)} \brck{A^\mu(\bs, a) A^t(\bs, a)} }
\end{align}

Now, for the last step we note that $d^{t+1}_\rho$ is component wise larger than $\rho$, and this gives us the final result:
\begin{align}
    \label{eq:proof-11}
    \E_{\bs\sim \rho}\brck{ V^{t+1}(\bs) -  V^{t}(\bs)} &\gsim
    \gamma   \E_{\bs\sim \rho}  \brck{\mathbb{V}_{a\sim \pi_t(a\mid \bs)} \brck{A^\mu(\bs, a)} - \E_{a\sim \pi_t(a\mid \bs)} \brck{A^\mu(\bs, a) A^t(\bs, a)} }
\end{align}\end{proof}

\subsection{Discussion on Remark~\ref{rem:bok-result}}
\label{subsec:remark-additional}

First, we note that if the $Q$-value of a base policy $\pi$ at state, action pair $(\bs, a)$ is $Q^\pi(\bs, a)$, then for the prover $\mu$ set to the best-of-K policy $\mathrm{BoK}(\pi)$, the $Q$-value at the same state, action pair is:
\begin{align}
    \label{eq:proof-25}
    Q^\mu(\bs, a) = 1-(1-Q^\pi(\bs, a))^K
\end{align}

This is because, in our setup the final outcome of correctness given by $\Rex$ is a random variable taking values in $\{0, 1\}$, with expectation $Q^\pi(\bs, a)$, when completing a rollout starting from prefix $(\bs, a)$. Thus, we can treat the outcome of function $\Rex$ as a Bernoulli random variable. Next, we can simply compute the probability of sampling a single correct answer out of $K$ attempts, which is also a Bernoulli random variable with mean given by $Q^\mu(\bs, a)$ in \Eqref{eq:proof-25}.   

Now for Remark~\ref{rem:bok-result}, we observe that whenever $Q^\pi(\bs, a) \ll 1$, \textit{e.g.}, when $Q^\pi(\bs, a) = O(1/K)$, for all values of $(\bs, a)$, we can do the Taylor approximation of $Q^\mu(\bs, a)$ around $0$, and note that $Q^\mu(\bs, a) = \Theta (K \cdot Q^\pi(\bs, a))$.
Next, note the following calculation for the first term (distinguishability) in Theorem~\ref{thm:policy-improvement}:
\begin{align*}
    \mathbb{V}_{\pi} A^\mu(\bs, a) &= \mathbb{V}_{\pi} Q^\mu(\bs, a) \\
    &= \mathbb{V}_{\pi} \brck{1-(1-Q^\pi(\bs, a))^K} \\
    &=  \mathbb{V}_{\pi} \brck{(1-Q^\pi(\bs, a))^K}
\end{align*}
This means that the first term in Theorem~\ref{thm:policy-improvement} which measures ``distinguishability'' now increases by a factor of $K^2$. Similarly, we can see that the the term which measures ``misalignment'' can change in magnitude by atmost a factor of $O(K)$, since the misalignment term is linear in $Q^\mu$. These two observations combined lead us to the conclusion  in Remark~\ref{rem:bok-result}.

\subsection{Improving a Stronger Base policy with a Weaker Prover Policy}
\label{subsec:prp-learning-signal-proof}

In Proposition~\ref{prp:learning-signal}, we 
consider the case where the $\pi$ and $\mu$ differ in performance, as measured under the distribution of states $\rho$ in the following way:
$$\E_{\bs\sim \rho} [| V^\mu(\bs) - V^{\pi_t}(\bs)|] = \eta.$$ 
Next, whenever the prover's preference over actions is complementary to the base policy, by a factor of $\eta$,  i.e., 
$$\E_{\bs\sim \rho} \Exp_{a\sim \pi}[
\abs{Q^\mu(\bs, a)- Q^{\pi_t}(\bs, a)}] = \Theta(\eta),$$
then the variance of $A^\mu$ or $A^{\pi_t}$ under $\pi_t$ should scale as $\eta^2$.

Thus, we see that when $\pi_t$ fails to distinguish actions (\textit{i.e.}, $\E_{\bs\sim \rho}\mathbb{V}_{\pi_t}[A^{\pi_t}(\bs, a)]$ is small) regardless of the strength of prover policy $\mu$, as long as it is sufficiently complementary to $\pi_t$, the prover policy induces an improvement in base policy, that scales as $\eta^2$.

\begin{proposition}[Complementary $\mu$ boosts improvements in $\pi$]
\label{prp:learning-signal} 
Under the distribution over states given by $\rho$, let prover $\mu$ and base policy at iterate $t$, $\pi_t$, differ in absolute performance, i.e.,  $$\E_{\bs\sim \rho} [| V^\mu(\bs) - V^{\pi_t}(\bs)|]=\eta.$$ When $\E_{\bs \sim \rho} \mathbb{V}_{a\sim \pi_t}[A^{\pi_t}(\bs, a)] < \E_{\bs \sim \rho} \mathbb{V}_{a\sim \pi_t}[A^\mu(\bs, a)]$, and $\mu$ is complementary to $\pi_t$, i.e., $$\E_{\bs \sim \rho}\E_{a\sim \pi_t}|Q^{\pi_t}(\bs, a) - Q^\mu(\bs_h, a)| = \Theta(\eta),$$ then $\E_{\bs \sim \rho} [V^{\pi_{t+1}}(\bs)-V^{\pi_t}(\bs)] \gsim \eta^2$.
\end{proposition}

\begin{proof}
 We begin by proving an upper bound on the disagreement between prover and base policy:
\begin{align*}
    \;\;\; &\E_{\bs\sim \rho}\E_{a\sim \pi_t}|Q^\mu(\bs, a) - Q^{\pi_t}(\bs, a)|  \\
     &\leq \E_{\bs\sim \rho}\E_{a\sim \pi_t}|Q^\mu(\bs, a) -V^\mu(\bs)| + \E_{\bs\sim \rho}\E_{a\sim \pi_t}|Q^{\pi_t}(\bs, a) -V^\mu(\bs)| \\
     &\leq \E_{\bs\sim \rho}\E_{a\sim \pi_t}|Q^\mu(\bs, a) -V^\mu(\bs)| + \E_{\bs\sim \rho}[ |V^\mu(\bs) - V^{\pi_t}(\bs)|] +  \E_{\bs\sim \rho}\E_{a\sim \pi_t}|Q^{\pi_t}(\bs, a) -V^{\pi_t}(\bs, a)|  \\
     &\leq \eta + \E_{\bs\sim \rho} \sqrt{\mathbb{V}_{\pi_t}[A^\mu(\bs, a)]} +  \E_{\bs\sim \rho} \sqrt{\mathbb{V}_{\pi_t}[A^{\pi_t}(\bs, a)]},
\end{align*}
where the last inequality uses Cauchy-Schwartz. Next we apply Jensen's inequality on the terms in the square root, and the conditions that $\E_{\bs \sim \rho}\E_{a\sim \pi_t}|Q^{\pi_t}(\bs, a) - Q^\mu(\bs_h, a)| = \Omega(\eta)$ to conclude:
\begin{align*}
      \sqrt{\E_{\bs\sim \rho} \mathbb{V}_{\pi_t}[A^\mu(\bs, a)]} \gsim \eta,
\end{align*}
whenever $\E_{\bs \sim \rho} \mathbb{V}_{a\sim \pi_t}[A^{\pi_t}(\bs, a)] < \E_{\bs \sim \rho} \mathbb{V}_{a\sim \pi_t}[A^\mu(\bs, a)]$. Now that we lower bound ``distinguishability'', it is easy to see that a similar derivation would upper bound the magnitude of the misalignment term by $O(\eta^2)$. Invoking the result in Theorem~\ref{thm:policy-improvement} yields:
\begin{align*}
    E_{\bs \sim \rho} [V^{\pi_{t+1}}(\bs)-V^{\pi_t}(\bs)] \gsim \eta^2
\end{align*}\end{proof}

\section{Examples Generated by Base Policy Trained on Rewards \texorpdfstring{$Q^\pi + \alpha Q^\mu$}{Qmu}}
\label{sec:q_value_failure}

When we train the base policy with reinforcement learning where the reward is $Q^\pi + \alpha Q^\mu$, instead of the effective reward $Q^\pi + \alpha A^\mu$, we find that the policy training converges to contrived solutions with degenerate failures. 

For example, we see in multiple examples (like the four below), that for every question, the model learns to output "REPHRASE THE PROBLEM" as the first step. 
This is because under $Q^\mu$ this trivial step achieves a positive reward. While there exists a better solution where the policy outputs the correct answer and achieves higher rewards, it fails to recover from this local optima and find the better solution. Across multiple training runs, we observed similar trends when training with $Q^\pi + \alpha Q^\mu$. 
Note that in the idealized update, where we observe advantages $A^\mu$ or value function $Q^\mu$ on all actions at a state (\textit{e.g.}, in the NPG update in \Secref{subsec:theory}),  using $Q^\pi + \alpha Q^\mu$ is equivalent to $Q^\pi + \alpha A^\mu$. But, in practice, $ A^\mu$ offers significant benefits in the form of variance reduction in policy gradient iterates. In this instance, it does that by not assigning reward to steps like ``REPHRASE THE PROBLEM'', since the likelihood of a prover solving before or after this step is expected to be roughly the same under any choice of the prover policy (\textit{i.e.}, the prover policy makes no progress).

\begin{tcolorbox}[colback=blue!5!white, colframe=blue!75!black, title=Example 1, breakable]
\scriptsize
\begin{verbatim}
Q: My club has 25 members.  In how many ways can I choose members to form a 4-person
executive committee?

A:

>>>
REPHRASE THE PROBLEM
WE NEED TO FIND THE NUMBER OF WAYS TO CHOOSE 4 MEMBERS OUT OF 25 FOR THE EXECUTIVE 
COMMITTEE SO WE CAN USE THE COMBINATION FORMULA WHICH IS 
N C R = N! / (R! * (N - R)!) WHERE N IS THE TOTAL NUMBER OF MEMBERS AND R IS 
THE NUMBER OF MEMBERS WE CHOOSE SO IN THIS CASE N = 25 AND R = 4 SO WE PLUG THEM IN AND 
SIMPLIFY 25 C 4 = 25! / (4! * (25 - 4)!) = (25 * 24 * 23 * 22) / (4 * 3 * 2 * 1) = 104005
# Answer

104005
\end{verbatim}
\end{tcolorbox}

\begin{tcolorbox}[colback=blue!5!white, colframe=blue!75!black, title=Example 2, breakable]
\scriptsize
\begin{verbatim}
Q: We have a triangle $\triangle ABC$ and a point $K$ on $BC$ such that $AK$ is
an altitude of $\triangle ABC$. If $AC = 10,$ $BK = 7$, and $BC = 13,$ then 
what is the area of $\triangle ABC$?

A:

>>>
REPHRASE THE PROBLEM
WE NEED TO FIND THE AREA OF $\triangle ABC$ SO WE CAN USE THE FORMULA $A = \frac{1}{2}bh$ 
WHERE $b$ IS THE BASE AND $h$ IS THE HEIGHT SO SINCE $AK$ IS AN ALTITUDE, THAT MEANS
$AK \perp BC$ SO WE CAN USE $BC$ AS THE BASE AND $AK$ AS THE HEIGHT SO NOW WE NEED 
TO FIND $AK$ SO WE CAN USE THE PYTHAGORE THEOREM ON $\triangle AKC$ SINCE $AC = 10$
# Answer

48
\end{verbatim}
\end{tcolorbox}

\begin{tcolorbox}[colback=blue!5!white, colframe=blue!75!black, title=Example 3, breakable]
\scriptsize
\begin{verbatim}
Q: Suppose that $n, n+1, n+2, n+3, n+4$ are five consecutive integers.

Determine a simplified expression for the sum of these five consecutive integers.

A:

>>>
REPHRASE THE PROBLEM
WE NEED TO FIND THE SUM OF THESE FIVE CONSECUTIVE INTEGERS SO WE CAN USE THE FORMULA 
$S_n = \frac{n}{2}(a_1 + a_n)$ WHERE $S_n$ IS THE SUM, $n$ IS THE NUMBER OF TERMS, 
$a_1$ IS THE FIRST TERM, AND $a_n$ IS THE LAST TERM SO IN THIS CASE $n = 5, a_1 = n, 
a_5 = n + 4$ 
# Answer

5n + 10
\end{verbatim}
\end{tcolorbox}

\begin{tcolorbox}[colback=blue!5!white, colframe=blue!75!black, title=Example 4, breakable]
\scriptsize
\begin{verbatim}

Q: What is the value of $x^2+y^2-z^2+2xy$ when $x=-3$, $y=5$, and $z=-4$?

A:

>>>
REPHRASE THE PROBLEM
WE NEED TO FIND THE VALUE OF $x^2+y^2-z^2+2xy$ WHEN $x=-3$, $y=5$, AND $z=-4$ 
SO WE JUST PLUG IN THE VALUES AND SIMPLIFY
# Answer

-2


\end{verbatim}
\end{tcolorbox}

\end{appendix}

\end{document}